\documentclass[twoside]{article}

\usepackage[accepted]{aistats2023}
\newcommand\blfootnote[1]{%
	\begingroup
	\renewcommand\thefootnote{}\footnote{#1}%
	\addtocounter{footnote}{-1}%
	\endgroup
}

\usepackage[round]{natbib}

\usepackage{algorithm}
\usepackage{algorithmic}

\usepackage{amsmath}
\usepackage{amssymb}
\usepackage{mathtools}
\usepackage{amsthm}

\usepackage{booktabs}       
\usepackage{multirow}
\usepackage{xcolor}      

\theoremstyle{plain}
\newtheorem{theorem}{Theorem}[section]
\newtheorem{proposition}[theorem]{Proposition}
\newtheorem{lemma}[theorem]{Lemma}
\newtheorem{corollary}[theorem]{Corollary}
\theoremstyle{definition}
\newtheorem{definition}[theorem]{Definition}
\theoremstyle{example}
\newtheorem{example}[theorem]{Example}

\theoremstyle{remark}
\newtheorem{remark}[theorem]{Remark}

\ifodd 0
\newcommand{\revc}[1]{{\color{blue}#1}}
\else
\newcommand{\revc}[1]{#1}
\fi

\ifodd 1
\newcommand{\revn}[1]{{\color{magenta}#1}}
\else
\newcommand{\revn}[1]{}
\fi

\ifodd 0
\newcommand{\rev}[1]{{\color{blue}#1}}
\else
\newcommand{\rev}[1]{}
\fi

\ifodd 0
\newcommand{\com}[1]{{\color{red}#1}}
\else
\newcommand{\com}[1]{}
\fi

\newcommand{\bs}{\boldsymbol}
\newcommand{\R}{\mathbb{R}}
\renewcommand{\Pr}{\mathbb{P}}
\newcommand{\E}{\mathbb{E}}

\DeclareMathOperator{\Int}{int}
\DeclareMathOperator{\cl}{cl}
\DeclareMathOperator{\bd}{bd}

\allowdisplaybreaks

\begin{document}

%

%

\twocolumn[

\aistatstitle{Vector Optimization with Stochastic Bandit Feedback}

\aistatsauthor{\c{C}a\u{g}{\i}n Ararat$^{*}$ \And Cem Tekin$^{*}$}

\runningauthor{\c{C}a\u{g}{\i}n Ararat, Cem Tekin}

\aistatsaddress{Bilkent University, Ankara, Turkey \And  Bilkent University, Ankara, Turkey} ]

\begin{abstract}
 We introduce vector optimization problems with stochastic bandit feedback, in which preferences among designs are encoded by a polyhedral ordering cone $C$. Our setup generalizes the best arm identification problem to vector-valued rewards by extending the concept of Pareto set beyond multi-objective optimization. 
 We characterize the sample complexity of ($\epsilon,\delta$)-PAC Pareto set identification by defining a new cone-dependent notion of complexity, called the {\em ordering complexity}. In particular, we provide gap-dependent and worst-case lower bounds on the sample complexity and show that, in the worst-case, the sample complexity scales with the square of ordering complexity. 
 Furthermore, we investigate the sample complexity of the naïve elimination algorithm and prove that it nearly matches the worst-case sample complexity. Finally, we run experiments to verify our theoretical results and illustrate how $C$ and sampling budget affect the Pareto set, the returned ($\epsilon,\delta$)-PAC Pareto set, and the success of identification.
\end{abstract}
\blfootnote{$^{*}$Equal contribution and alphabetical ordering.}

\section{INTRODUCTION}\label{sec:intro}

We consider a vector optimization problem with $D$ objectives and $K$ designs, where $D$ and $K$ are positive integers. The mean vector of design $i\in [K]$ is given by $\bs{\mu}_i \in \R^D$. The mean vectors are partially ordered based on an ordering cone $C \subseteq \R^D$. When these vectors are known, the set of all Pareto optimal designs (Pareto set) can be found by solving the following optimization problem \citep{jahn}
\begin{align}
	& \text{maximize } \bs{\mu}_i \text{ with respect to cone } C \text{ over} ~i \in [K] ~. \label{eqn:vectoropt}
\end{align}

In this paper, for the first time in the literature, we investigate vector optimization problems in a pure exploration setting with stochastic bandit feedback, i.e., when $\bs{\mu}_i$s are unknown and the query of each design $i$ only yields a noisy observation of $\bs{\mu}_i$. In particular, we seek an answer to the following fundamental question: {\em What is the sample complexity of Pareto set identification and how does it depend on $C$?}

In Section \ref{sec:motivation}, by comparing vector optimization with scalar and multi-objective optimization within the context of pure exploration, we argue why one should consider solving problems like \eqref{eqn:vectoropt}. In Section \ref{sec:contibutions}, we detail our main contributions.

\subsection{Motivation}\label{sec:motivation}

Pure exploration problems have attracted significant interest from both machine learning theorists and practitioners \citep{even2006action, bubeck2009pure, audibert2010best, kalyanakrishnan2012pac, karnin2013almost, kaufmann2013information, zhou2014optimal, laskey2015multi}. The simplest pure exploration problem takes the form of a $K$-armed bandit with arm means $\mu_1,\ldots,\mu_K\in \mathbb{R}$, in which the learner aims to identify an optimal arm (aka design) by sequential experimentation under noisy feedback. In this setup, a noisy scalar reward is revealed to the learner immediately after the selection of a design. As experimentation consumes resources, the learner seeks to minimize the number of evaluations by adapting its arm selection based on past reward observations. This problem has been formalized in many different ways, with notable examples including $(\epsilon,\delta)$-probably approximately correct (PAC)  \citep{even2006action} (aka Explore-1), fixed confidence and fixed budget \citep{karnin2013almost} best arm identification. Almost all research in this field focuses on devising sample-efficient algorithms that ``beat the noise" by adaptive sampling and elimination. Algorithmic contributions include a plethora of techniques such as Median Elimination \citep{even2006action}, LUCB \citep{kalyanakrishnan2012pac}, KL-LUCB \citep{kaufmann2013information}, Sequential Halving \citep{karnin2013almost}, Track-and-Stop \citep{garivier2016optimal}, ABA \citep{hassidim2020optimal}, $\epsilon$-TaS \citep{garivier2021nonasymptotic}, several Bayesian sampling strategies \citep{qin2017improving,shang2020fixed,russo2020simple}, and saddle-point methods \citep{degenne2019non,degenne2020gamification}. The difficulty of the problem depends on the suboptimalty gap of each arm $i$, straightforwardly found as $\max_j{\mu_j} - \mu_i$. Without noise, everything turns into a simple argmax operation over the design set. 

However, not many real-world problems naturally exhibit scalar rewards. Optimization of $D$-dimensional performance metrics ($D>1$) is necessary for tasks such as hardware design \citep{zuluaga2016varepsilon}, and clinical trials for drug development and dose identification \citep{lizotte2016multi}. Nevertheless, vector-valued objectives $\bs{\mu}_i$ can be scalarized by weights $\bs{w}$ that encode the importance of each objective, turning the learning problem into a scalar one.  The choice of $\bs{w}$ is often left to the practitioner, and the right choice might be difficult to come up with.

To tackle this issue, another strand of literature focuses on identifying the set $P^\ast$ of all Pareto optimal designs, whose mean vectors $\bs{\mu}_i$ are not dominated by others' mean vectors. In particular, when $D=2$, $\bs{\mu}_i=(\mu_i^1,\mu_i^2)^{\mathsf{T}}$ is dominated by $\bs{\mu}_j=(\mu_j^1,\mu_j^2)^{\mathsf{T}}$ ($\bs{\mu}_i \preceq \bs{\mu}_j$) if and only if (iff) $\forall w \in [0,1] :  w \mu^1_i + (1-w) \mu^2_i \leq  w \mu^1_j + (1-w) \mu^2_j$, which is equivalent to the usual componentwise order on $\R^2$. This multi-objective optimization problem has been extensively studied in the pure exploration setting with stochastic bandit feedback. \citet{auer2016pareto} study the sample complexity of Pareto set identification in stochastic $K$-armed bandit problems. \citet{hernandez2016predictive}, \citet{shah2016pareto}, \citet{zuluaga2016varepsilon} study Pareto set identification in problems with large design sets. They use Gaussian processes to capture correlations between designs such that a bulk of designs can be explored by a single sample. In another related work,  \citet{katz2018feasible} propose the feasible arm identification problem, in which the goal is to identify arms whose mean rewards belong to a given polyhedron through noisy evaluations.

Apart from pure exploration, multi-objective learning has also been extensively investigated in the bandit regret minimization setting \citep{drugan2013designing,turgay2018multi}. This line of research defines gaps that represent the suboptimality of arms, and propose algorithms that minimize the cumulative sum of gaps of the selected arms while ensuring some sort of balance between selections of Pareto optimal arms. Another related line of research is multi-objective reinforcement learning \citep{van2014multi, hayes2022practical}.

\begin{figure}[h!]
	\centering
	\includegraphics[width=0.7\linewidth]{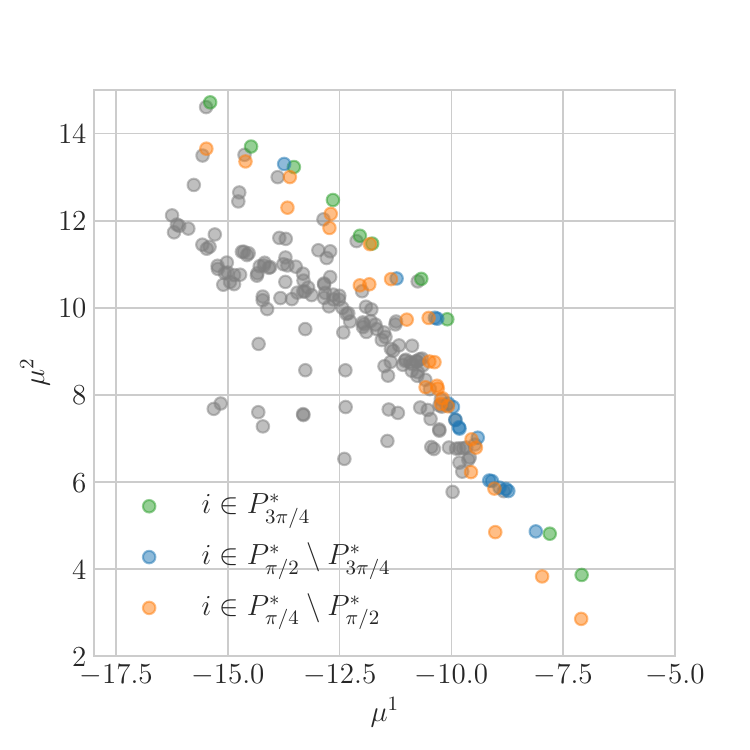}
	\caption{Pareto sets $P^*_{\pi/4}$, $P^*_{\pi/2}$, $P^*_{3\pi/4}$ of SNW dataset used in experiments (Section \ref{sec:experiment}) for the cones $C_{\pi/4}$, $C_{\pi/2}$, $C_{3\pi/4}$  (Example \ref{ex:2Dsmall}). By the structure of the cones, $P^*_{3\pi/4} \subseteq P^*_{\pi/2} \subseteq P^*_{\pi/4}$.}\label{fig:paretosets}
\end{figure}

\begin{figure}
	\centering
	\includegraphics[width=0.6\linewidth]{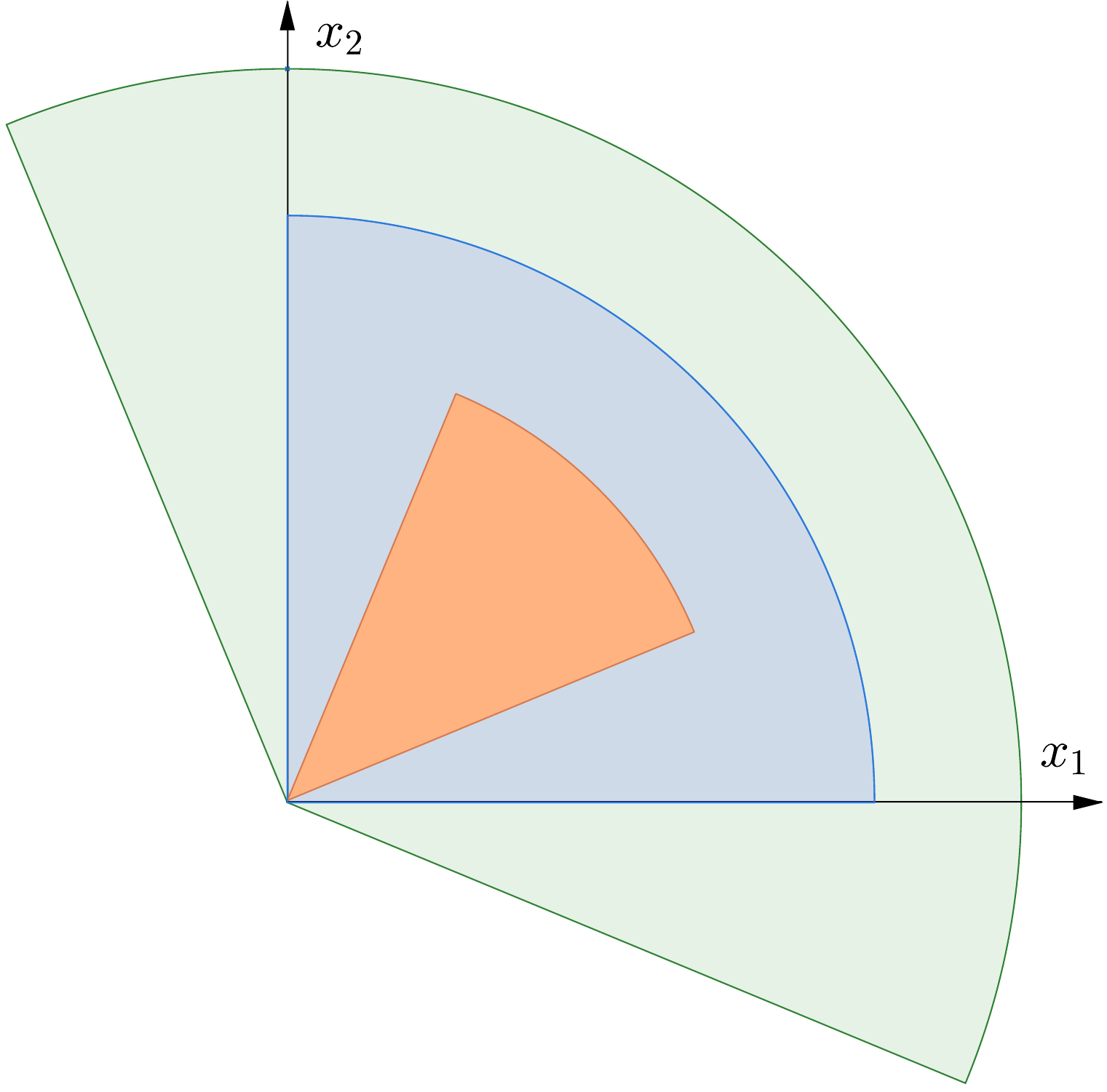}
	\caption{Three cones in $\mathbb{R}^2$ (intersected with balls for illustration purposes). Orange, blue, green cones correspond to $C_{\pi/4}$, $C_{\pi/2}$, $C_{3\pi/4}$ defined in Example \ref{ex:2Dsmall}, and used in experiments in Section \ref{sec:experiment}. $C'$ corresponds to a wide cone like the green one, i.e., $C' \supseteq \mathbb{R}^2_{+}$.}\label{fig:cones}
\end{figure}

While the multi-objective optimization viewpoint saves the practitioner from fixing a weight vector for the objectives, it ends up returning a large set of designs (see, e.g., $P^*_{\pi/2}$ in Figure \ref{fig:paretosets}), which might be even more frustrating. Moreover, the practitioners may want to narrow down the set of alternatives by using their domain knowledge about the relative importance of each objective. For instance, if the practitioner wants to give at least $100\alpha\%$ relative importance to each objective for $\alpha \in (0,0.5)$, then this can be achieved by defining a new partial order $\preceq_{C'}$ with $\bs{\mu}_i \preceq_{C'} \bs{\mu}_j$ iff 
$\forall w \in [\alpha,1-\alpha] :  w \mu^1_i + (1-w) \mu^2_i \leq  w \mu^1_j + (1-w) \mu^2_j$. This partial order is equivalent to a partial order induced by the cone\footnote{See Figure \ref{fig:cones} for examples of ordering cones in $\mathbb{R}^2$, and Definition \ref{defn:cone} for the precise definition of this order.}
\begin{align}
	C' = \{ \bs{x} \in \mathbb{R}^2\mid \alpha x_1 & + (1-\alpha) x_2 \geq 0 , \notag \\
	& (1-\alpha)  x_1 + \alpha x_2 \geq 0 \} ~. \label{eqn:cprime}
\end{align}
Such a choice usually narrows down the set of all Pareto optimal designs (see, e.g., $P^*_{3\pi/4}$ in Figure \ref{fig:paretosets}). In addition, the flexibility of specifying the relative importance of each objective leads to a wide spectrum of Pareto sets that cannot be captured by multi-objective optimization. The aim of this work is to propose a new framework for pure exploration problems in which the preferences of practitioners are encoded by ordering cones.

One real-world example in which multi-objective optimization does not suffice is intensity-modulated radiation therapy (IMRT) planning \citep{teichert2019targeted}. Here, the designs represent treatment strategies (e.g., dosage). The objectives are related to the effectiveness of therapy against the tumor and the damage to the surrounding tissue, which are in conflict with each other. Based on the relative importance of each objective encoded by the cone parametrized by $\alpha \in (0, 0.5)$, the practitioner might want to identify a set of Pareto optimal treatments via computer simulations. Note that, in this application, returning the Pareto set of designs according to the componentwise order can be harmful for the patient. For instance, a treatment which achieves a large reduction in tumor size at the cost of extensive damage to the surrounding healthy tissue could be returned if componentwise order were used.

\begin{example}
	We illustrate the limitation of using the componentwise order in IMRT via a simple example. Take $D=2$, $K=4$. The first dimension represents the effectiveness against the tumor. Assume that the best value is $1$ (e.g., the tumor is completely destroyed) and the worst value is
	$0$ (e.g., the tumor does not shrink at all). The second dimension represents the negative of the damage to the surrounding tissue. Let the best value be $0$ (e.g., the surrounding tissue is not damaged at all) and the worst value be $-1$ (e.g., the surrounding tissue is completely destroyed). Let $\bs{\mu}_1 = [0.8, -0.3]^{\mathsf{T}}$, $\bs{\mu}_2 = [0.9, -0.5]^{\mathsf{T}}$, $\bs{\mu}_3 = [0, 0]^{\mathsf{T}}$, $\bs{\mu}_4 = [1, -1]^{\mathsf{T}}$. The Pareto set under the componentwise order (i.e., the classical multi-objective case) includes all four designs for this setting. Design $3$ is useless since it has no effect on the tumor. Design $4$ is the most effective against the tumor but it causes significant damage to the surrounding tissue; thus, it should not be returned as a treatment option. By returning all four designs, the componentwise order does not provide any useful guidance to the practitioner. On the other hand, the Pareto set with cone $C$ parameterized by relative importance parameter $\alpha \approx 0.29$ (equivalent to cone with $\theta = 3\pi/4$ in Example \ref{ex:2Dsmall} below) includes only designs $1$ and $2$, both of which provide meaningful compromise between effectiveness and side-effects.
\end{example}

\subsection{Contributions} \label{sec:contibutions} 

The subfield of mathematical optimization which generalizes scalar and multi-objective optimization via partial orders induced by cones is called vector optimization \citep{jahn,lohne}. While there exists a plethora of works that focuses on multi-objective optimization through noisy evaluations \citep{auer2016pareto,hernandez2016predictive,zuluaga2016varepsilon,shah2016pareto}, our work is the first to consider vector optimization for partial orders in $\mathbb{R}^D$ induced by general polyhedral ordering cones within the stochastic $K$-armed bandit framework. 

Our first contribution is to generalize the ($\epsilon,\delta$)-PAC best arm identification problem in \citet{even2006action} to vector optimization (Definition \ref{defn:PAC}). Our success condition takes into account the directions of improvement specified by $C$ and the uncertainty induced by the norm-subgaussian sampling noise. In line with PAC best arm identification, instead of trying to identify all Pareto optimal designs, we introduce a direction-free, cone-dependent $\epsilon$-covering requirement for the Pareto set. Based on this, we define the $(\epsilon,\delta)$-PAC Pareto set identification problem. 

Our second contribution is to define the notion of {\em ordering complexity} {$\beta$}, determined by $C$ (equations \eqref{eq:Csup1}, \eqref{eq:Csup2}). The ordering complexity of $C$ characterizes the sampling budget required for ($\epsilon,\delta$)-PAC Pareto set identification. We explain how it can be computed for any polyhedral ordering cone $C$ ({Theorem \ref{thm:beta} in the main text and the theorems 
	in the supplemental document}).

Our third contribution is to characterize the learning difficulty of a vector optimization problem by identifying fundamental gaps associated with the geometric properties of the ordering cone $C$ (\revc{equations \eqref{eq:m}, \eqref{eq:M}}). We show that when mean vectors are known, these gaps can be calculated by solving convex optimization problems involving quadratic and affine functions (\revc{Propositions \ref{prop:m}, \ref{prop:M}}). Motivated by the fact that the sampling noise can shift mean estimates in arbitrary directions, these gaps extend the direction-based gaps proposed for multi-objective problems in \citet{auer2016pareto} to arbitrary polyhedral ordering cones and arbitrary directions. In particular, we rely on the interpretation of the ordering cone as the set of all directions of improvement, and define two gaps between designs $i,j$: $m(i,j)$ as the minimum step-length in an arbitrary direction of improvement for $i$ that is sufficient to avoid $i$ being dominated by $j$, $M(i,j)$ as the minimum step-length in an arbitrary direction of improvement for $j$ that is sufficient to make $j$ dominate $i$. 
For the special case of multi-objective problems, while $m(i,j)$ coincides with its analogue in \citet{auer2016pareto}, the analogue of $M(i,j)$ in \citet{auer2016pareto} gives higher values since the improvement is restricted to the diagonal direction.\footnote{When $C = \mathbb{R}^D_+$, our $m(i,j)$ is the same as \citet{auer2016pareto} but our $M(i,j)$ gives $\|(\bs{\mu}_j-\bs{\mu}_i)^-\|_2=\|(\bs{\mu}_i-\bs{\mu}_j)^+\|_2$ (by 
	Proposition \ref{prop:M}(ii)), and it is larger than $\|(\bs{\mu}_i-\bs{\mu}_j)^+\|_\infty$, which is the value of $M(i,j)$ in \citet{auer2016pareto}. In the worst-case, their $M(i,j)$ is $\sqrt{D}$ factor more conservative than ours.} Hence, our definition of $M(i,j)$ provides a more flexible gap calculation even in the multi-objective case.

Our fourth contribution is to provide gap-dependent and worst-case lower bounds on the sample complexity. Our lower bounding techniques involve constructing modified problems via perturbations of mean reward vectors along special directions in the objective space related to the geometry of the ordering cone. This makes our lower bound analysis substantially different from the analysis of scalar or multi-objective problems. In particular, we prove that the worst-case sample complexity scales as $\Omega((K \beta_2^2/\epsilon^2) \log(1/\delta))$, where $\beta_2$ is one of the two constants whose maximum define the ordering complexity.

As our final contribution, we take a first step towards Pareto set identification in \revc{vector optimization (beyond multi-objective optimization)} by proposing a na\"{i}ve elimination algorithm. In particular, we perform sample complexity analysis under the general norm-subgaussian noise assumption. We show that our algorithm successfully identifies an $(\epsilon,\delta)$-PAC Pareto set with a sample complexity of $O((K \beta^2/\epsilon^2) \log(1/\delta))$ (Theorem \ref{thm:samplecomp}). In addition, we provide numerical results for parametrized polyhedral cones, and establish a connection between the success rate of Pareto set identification and the angular width of the cone. The details of all proofs are given in the supplemental document.

\section{THE ORDERING CONE}\label{sec:prelim}

In this section, we fix the notation for the rest of the paper, introduce the order structure for vector optimization, and state some important properties of this structure. Let $D$ be a positive integer. We write $[D] := \{1,\ldots,D\}$ and denote by $\R^D$ the $D$-dimensional Euclidean space. The elements of $\R^D$ are denoted by boldface letters. For a vector $\bs{v}\in\R^D$, $||\bs{v}||_2$ represents its $\ell_2$ norm. For subsets $A$, $A'$ of $\mathbb{R}^D$, $A+A'$ and $A-A'$ denote their Minkowski sum and difference, respectively; $\cl(A)$, $\Int(A)$, $\bd(A)$, $A^c$ denote the closure, interior, boundary, and complement of $A$ in $\R^D$, respectively. For $\bs{v} =(v^1,\ldots,v^D)^{\mathsf{T}}\in \mathbb{R}^D$, we write $\bs{v}^+:=((v^1)^+,\ldots,(v^D)^+)^{\mathsf{T}}$ and $\bs{v}^-:=((v^1)^-,\ldots,(v^D)^-)^{\mathsf{T}}$, where $r^+:=\max\{0,r\}$ and $r^-:=-\min\{0,r\}$ for $r\in\R$. For $\bs{v}\in\R^D$ and $r\geq 0$, $B(\bs{v},r)$ represents the ball in $\mathbb{R}^D$ with center $\bs{v}$ and radius $r$. The distance of a vector $\bs{v} \in \mathbb{R}^D$ to a set $A \subseteq \mathbb{R}^D$ is defined as $d(\bs{v}, A) := \inf_{\bs{x} \in A} || \bs{v} - \bs{x} ||_2$. 

We consider a vector optimization problem with $D$ objectives and a finite set $[K]$ of designs, where $D, K$ are positive integers. The mean vector of design $i\in [K]$ is denoted by $\bs{\mu}_i = (\mu^1_i, \ldots, \mu^D_i)^{\mathsf{T}} \in \R^D$. To compare mean vectors, we will introduce a partial order on $\R^D$ based on an ordering cone. The latter notion is recalled next.

\begin{definition}\label{defn:orderingcone}
	A set $C\subseteq\mathbb{R}^D$ is called a \emph{cone} if $\lambda \bs{v}\in C$ for every $\bs{v}\in C$ and $\lambda\geq 0$. A cone $C$ is called \emph{pointed} if $C\cap (-C)=\{\bs{0}\}$, it is called \emph{solid} if $\Int(C)\neq\emptyset$. A closed convex cone that is pointed and solid is called an \emph{ordering cone} (proper cone).
\end{definition}

Let $C\subseteq\R^D$ be an ordering cone. As a consequence of the convexity of $C$, it is immediate that $C+C=C$. In general, an ordering cone can be polyhedral (e.g., the positive orthant $C=\R^D_+$) or non-polyhedral (e.g., the ice-cream cone $C=\{\bs{x}\in\R^D\mid \|(x^1,\ldots,x^{D-1})^{\mathsf{T}}\|_2\leq x^D\}$ for $D\geq 2$). For the purposes of this paper, we focus on the polyhedral case as detailed in the next definition.


\begin{definition}\label{defn:polycone}
	A cone $C$ is called \emph{polyhedral} if it can be written as $C=\{\bs{x}\in\R^D\mid W\bs{x}\geq 0\}$ for some $N\times D$ real matrix $W$ with rows $\bs{w}^{\mathsf{T}}_1, \ldots, \bs{w}^{\mathsf{T}}_N$, and positive integer $N$.
\end{definition}

Figure \ref{fig:cones} illustrates several examples of polyhedral ordering cones in $\mathbb{R}^2$. Throughout the paper, we assume that $C$ is a polyhedral ordering cone as in Definition \ref{defn:polycone}.\footnote{Cone $C$ represents the preferences of the practitioner. It is treated as an input.} Without loss of generality, we assume that this description has no redundancies, that is, $W$ has the minimal number of rows; as well as that $|| \bs{w}_n ||_2 = 1$ for each $n \in [N]$.\footnote{Unless stated otherwise, all of the results in this paper hold when $C$ is a polyhedral ordering cone without redundancy and $|| \bs{w}_n ||_2 = 1$ for each $n \in [N]$.} It follows that the interior of $C$ is given by
$\Int(C) = \{ \bs{x} \in \mathbb{R}^D \mid  W \bs{x} > \bs{0} \}$.
When all entries of $W$ are nonnegative, it is clear that $C \supseteq \mathbb{R}^D_{+}$. In financial mathematics, such cones have found applications in multi-asset markets with transaction costs as ``solvency cones'', where $W$ is defined in terms of the bid-ask prices of the assets; see \citet{kabanov}.

\begin{remark}
	As discussed in Section~\ref{sec:motivation}, practitioners can prefer large cones ($C \supset \mathbb{R}^D_+$) to narrow down the set of Pareto optimal designs. On the other hand, some applications can benefit from using small cones ($C \subset \mathbb{R}^D_+$). One motivating example is the small molecule drug discovery problem \citep{jayatunga2022ai}. Consider the optimization of $D$ properties such as solubility, metabolic stability, toxicity. Assume that the final goal is to identify the Pareto optimal molecules in $P^*_{\pi/2}$ from a set of candidates $[K]$ according to cone $\mathbb{R}^D_+$. The discovery process starts with (relatively cheaper) in silico experiments followed by (more expensive) wet lab experiments. Due to distribution mismatch, running in silico experiments with $C = \mathbb{R}^D_+$ could result in failing to return an $(\epsilon,\delta)$-PAC set for $C = \mathbb{R}^D_+$. By using a smaller cone $C \subset \mathbb{R}^D_+$ for in silico experiments, the practitioner can ensure that more designs are passed to the wet lab stage $\hat{P}_{C} \supseteq \hat{P}_{\pi/2}$ and, since all Pareto optimal designs according to $\mathbb{R}^D_+$ are still in $C$, one can perform more expensive wet lab experiments on these designs to ensure that an $(\epsilon,\delta)$-PAC set is returned. In this case, the knowledge of an upper bound on the distribution mismatch can be used to select $C$.
\end{remark}


In order to characterize the sample complexity of Pareto set identification, we introduce the following constants that are related to the geometry of $C$: 
\begin{align}
	&\beta_1:=\sup_{\bs{x}\notin C}\frac{d(\bs{x},C\cap(\bs{x}+C))}{d(\bs{x},C)}, \label{eq:Csup1} \\
	& \beta_2:=\sup_{\bs{x}\in \Int(C)}\frac{d(\bs{x},(\Int(C))^c\cap(\bs{x}-C))}{d(\bs{x},(\Int(C))^c)} ~.
	\label{eq:Csup2}
\end{align}
These constants, which will appear in the crux of our sample complexity analysis, are novel to our work. We will provide intuitive explanations for these constants after defining fundamental gaps associated with Pareto set identification in Section \ref{sec:pareto} (see Remark~\ref{rem:intuitive}). In this section, we analyze their technical properties.
Note that we have $\beta_1\geq 1$ and $\beta_2\geq 1$. Let $\beta:= \max\{\beta_1,\beta_2\}$. We call $\beta$ the \emph{ordering complexity} of $C$.\footnote{$\beta$ only depends on $C$. It does not depend on the mean rewards.} The following theorem summarizes its main properties. To that end, for each $n\in [N]$, let us introduce the constant $\alpha_n:=\sup_{\bs{u}\in B(\bs{0},1)\cap C}\bs{w}_n^{\mathsf{T}}\bs{u}$; note that $\alpha_n\in (0,1]$ since we assume that $\sup_{\bs{u}\in B(\bs{0},1)}\bs{w}_n^{\mathsf{T}}\bs{u}= \|\bs{w}_n\|_2=1$.

\begin{theorem}\label{thm:beta}
	\revc{(i) When $C$ is a polyhedral ordering cone without redundancy and $|| \bs{w}_n ||_2 = 1$ for each $n \in [N]$, it holds $\beta_1<+\infty$ and $\beta_2\leq (\min_{n\in[N]}\alpha_n)^{-1}< +\infty$. (ii) Suppose further that $C\supseteq \R^D_+$. Then, $\beta_1=\beta_2=1$.}
\end{theorem}

Note that Theorem \ref{thm:beta} establishes the finiteness of the ordering complexity $\beta$. In the sample complexity analysis, we will assume that $\beta$ is known. Hence, we are also interested in the calculation of (an upper bound) $\beta$. While Theorem \ref{thm:beta}(i) already gives an easy-to-calculate upper bound on $\beta_2$, the calculation of $\beta_1$ is more involved and requires maximizing a Lipschitz function over a compact set. Moreover, in Theorem \ref{thm:beta}(ii), one can relax the condition $C\supseteq \R^D_+$ slightly. These technical points are discussed in detail with proofs and additional results in the supplemental document. 
\begin{remark}\label{rem:betacalc}
	The algorithm we propose in Section \ref{sec:alg} requires $\beta$ as input, which can be computed offline before starting evaluations. In addition, we would like to point out that the calculation of $\beta$ for large values of $D$ is not of primary concern in our setting. In both multi-objective and vector optimization, usually a small number $2 \leq D \leq 5$, and very typically only $D=2$ or $D=3$ of conflicting objectives are considered \citep{kovavcova2021time, umar}. While it is valuable to provide the decision-maker with a whole frontier of Pareto optimal designs for small number of objectives, the Pareto set becomes intractable for larger number of objectives. In the latter case, it would be more of a burden than flexibility for the decision-maker to choose from the Pareto set of designs.
\end{remark}

We calculate $\beta_1,\beta_2$ for some standard ordering cones in $\R^2$ in the next example.


\begin{example}\label{ex:2Dsmall}
	Let $D=2$. Given $\bs{x}\in\R^2$, let $\alpha(\bs{x})\in [0,2\pi)$ denote the angle in the polar coordinates of $\bs{x}$. Let $\theta\in (0,\pi/2]$ and define the ordering cone $C_\theta:=\{\bs{x}\in\R^2 \mid \alpha(\bs{x})\in [\pi/4-\theta/2,\pi/4+\theta/2]\}$. Let $\bs{x}\notin C_\theta$. Using elementary planar geometry, it can be checked that $\beta_1=\beta_2=\csc(\theta)$. The details of the calculation are given in the supplemental document. 
	If we take $\theta\in (\pi/2,\pi)$, then $\beta_1=\beta_2=1$ by Theorem \ref{thm:beta}(ii). Note that, up to a rotation of the angle bisector, this is the general form of an ordering cone for $D=2$; the current choice of the angle bisector is for convenience. 
\end{example}

\begin{remark} \label{remark:cones}
	Cone $C'$ in \eqref{eqn:cprime} can also be represented in the form given in Example \ref{ex:2Dsmall}. For instance, $C_{3\pi/4}$ in Figure~\ref{fig:paretosets} is such a cone with $\alpha = \frac{- \sin(\frac{\pi}{4} - \frac{\theta}{2})}{\sqrt{2} \sin( \frac{\theta}{2})} \approx 0.29$, where $\theta = 3 \pi /4$. Indeed, given any relative importance level $\alpha \in (0, 0.5)$ specified by the practitioner, one can find $\theta$ such that $\frac{- \sin(\frac{\pi}{4} - \frac{\theta}{2})}{\sqrt{2} \sin( \frac{\theta}{2})} = \alpha$.
\end{remark}



The polyhedral ordering cone $C$ induces two non-total order relations on $\R^D$ as defined next.

\begin{definition}\label{defn:cone}
	For every $\bs{\mu}, \bs{\mu}' \in \mathbb{R}^D$, we write $\bs{\mu} \preceq_{C} \bs{\mu}'$ if $\bs{\mu}'  \in \bs{\mu} + C$, and we write $\bs{\mu} \prec_{C} \bs{\mu}'$ if $\bs{\mu}'  \in \bs{\mu} + \Int(C)$.
\end{definition}

It can be checked that both $\preceq_{C}$ and $\prec_{C}$ are partial order relations on $\R^D$. If $C=\R^D_+$, then $\preceq_{C}$ coincides with the usual componentwise order on $\R^D$, which is the partial order used in multi-objective optimization. To explain the motivation for using an ordering cone that is different from the positive orthant, let us first observe that
\begin{equation}\label{eq:order}
	\bs{\mu}\preceq_C \bs{\mu}'\quad \Leftrightarrow \quad \forall n\in[N]\colon \bs{w}_n^{\mathsf{T}}\bs{\mu}\leq \bs{w}_n^{\mathsf{T}}\bs{\mu}'
\end{equation}
for every $\bs{\mu},\bs{\mu}'\in\R^D$. We denote by $C^+$ the convex cone generated by the rows of $W$, that is,
\begin{equation}\label{eq:dualcone}
	C^+:=\left\{\sum_{n=1}^N \lambda_n \bs{w}_n\mid \lambda_1,\ldots,\lambda_N\geq 0\right\}.
\end{equation}
It can be shown that $C^+$ is also a polyhedral ordering cone and it coincides with the so-called \emph{dual cone} of $C$, that is, $C^+=\{\bs{w}\in \R^D \mid \forall \bs{x}\in C\colon \bs{w}^{\mathsf{T}}\bs{x}\geq 0\}$. It is well-known that the positive orthant is self-dual, that is, $(\R^D_+)^+=\R^D_+$. By \eqref{eq:dualcone}, we may rewrite \eqref{eq:order} as
\begin{equation}\label{eq:order2}
	\bs{\mu}\preceq_C \bs{\mu}'\quad \Leftrightarrow \quad \forall w\in C^+\colon \bs{w}^{\mathsf{T}}\bs{\mu}\leq \bs{w}^{\mathsf{T}}\bs{\mu}'.
\end{equation}
Here, each $\bs{w}\in C^+$ can be considered as a ``weight vector'' since the partial order is determined by comparing the weighted sums of the components of $\bs{\mu}, \bs{\mu}'$ for all $\bs{w}\in C^+$. Moreover, if $C'$ is another ordering cone, then we have $C\supseteq C'$ iff $C^+\subseteq (C')^+$. In particular, if $W$ has non-negative entries so that $C\supseteq \R^D_+$, then we have $C^+\subseteq \R^D_+$. Consequently, $\preceq_C$ is weaker than the componentwise order in this case, thanks to \eqref{eq:order2}. In other words, in vector optimization, the decision-maker may have a more relaxed requirement when comparing two vectors, which is encoded by having a smaller set $C^+$ of weight vectors.

The above orders on $\R^D$ induce further orders on the design space $[K]$.

\begin{definition}\label{defn:partialorder}
	Let $i,j\in[K]$. Design $i$ is said to be \emph{weakly dominated} by design $j$ ($i \preceq_{C} j$) if $\bs{\mu}_i \preceq_{C} \bs{\mu}_j$. Design $i$ is said to be \emph{dominated} by design $j$ ($i \preceq_{C \setminus \{ \bs{0} \}}j$) if $\bs{\mu}_i \preceq_{C \setminus \{ \bs{0} \}} \bs{\mu}_j$. Design $i$ is \emph{strongly dominated} by design $j$ ($i \prec_{C} j$) if $\bs{\mu}_i \prec_{C} \bs{\mu}_j$. 
\end{definition}

\section{THE VECTOR OPTIMIZATION PROBLEM}\label{sec:optimization}

Equipped with the definitions in Section~\ref{sec:prelim}, the vector optimization problem can be expressed as
\begin{align*}
& \text{maximize } \bs{\mu}_i \text{ with respect to cone } C \text{ over }i \in [K] ~.
\end{align*}
The solution of this problem is called the Pareto set, as we define next.

\begin{definition}\label{defn:paretooptimal}
	A design $i\in [K]$ is called \emph{Pareto optimal} if it is not dominated by any other design with respect to the ordering induced by $C$. The \emph{Pareto set} $P^\ast := \{i \in [K] \mid \nexists j\in [K]\colon i \preceq_{C \setminus \{ \bs{0} \}} j \}$ consists of all Pareto optimal designs. 
\end{definition}

In vector optimization, a Pareto optimal design $i\in P^\ast$ is sometimes called an efficient solution or a maximizer of the design space $[K]$, and the corresponding objective vector $\bs{\mu_i}$ is called a $C$-maximal element of the objective set $\{\bs{\mu}_j\mid j\in[K]\}$; see \citet[Definition 3.1]{jahn}, \citet[Definition 2.1]{heydelohne}.

\begin{remark}\label{rem:Cpareto}
	Let us write $P^\ast(C):=P^\ast$ to emphasize the dependence on the ordering cone. Then, switching to a larger ordering cone $C^\prime\supseteq C$ makes the dominance relation weaker. Consequently, the Pareto set becomes smaller under the larger cone: $P^\ast(C^\prime)\subseteq P^\ast(C)$.
\end{remark}

	

\section{PARETO SET IDENTIFICATION}\label{sec:pareto}

We consider a learning problem where the mean vectors $\bs{\mu}_i$, $i \in [K]$, are not known beforehand. Our goal is to identify $P^*$ from noisy observations of the objective values of chosen designs in as few evaluations as possible. We assume that the total number of evaluations, $T\geq 1$, is finite. Evaluations are done in a sequential manner with $t\in [T]$ representing the round in which the $t$th evaluation is made, and we denote by $I_t$ the random variable representing the design evaluated in round $t$. We consider a noisy feedback setting in which the evaluation in round $t$ yields a random reward vector {$\bs{X}_{t}=\bs{\mu}_{I_t}+\bs{Y}_{I_t,t}$, where $\bs{Y}_{i,t}$ is the random noise vector associated to this evaluation when the evaluated design is $I_t=i$. We assume that the family $(\bs{Y}_{i,t})_{i\in[K], t\in[T]}$ consists of independent and centered random vectors, and it is independent of the family $(I_t)_{t\in[T]}$. We also assume that all members of this family are} norm-subgaussian with a common parameter $\sigma\geq 0$ as stated in the next definition.


\begin{definition}\label{defn:normsubgauss} \citep[Definition~3]{jin2019short}
	A centered random vector $\bs{Y}$ is called \emph{norm-subgaussian} with parameter $\sigma\geq 0$ if 
	$\Pr\{ \| \bs{Y} \|_2 \geq \epsilon\} \leq 2 e^{-\frac{\epsilon^2}{2\sigma^2}}$
	for every $\epsilon \geq 0$. 
\end{definition}

Examples of norm-subgaussian random vectors include bounded random vectors
and subgaussian random vectors (up to a scaling of the parameter); see \citet[Lemma 1]{jin2019short}.

Due to sampling noise, the sample complexity of identifying $P^\ast$ depends on the hardness of distinguishing the designs in $P^\ast$ from the designs that are not in $P^\ast$. We quantify the hardness by the following gaps.

Given two designs $i,j \in [K]$, we define $m(i,j)$ as the minimum increment in $\bs{\mu}_i$ in an arbitrary direction of increase that makes design $i$ not strongly dominated by design $j$. Formally, we have
\begin{align}\label{eq:m}
	m(i,j) :=\inf\{ s\geq 0\mid \exists & \bs{u}\in B(\bs{0},1)\cap C\colon \notag \\
	& ~ \bs{\mu}_i + s\bs{u} \notin \bs{\mu}_j-\Int(C)\} ~.
\end{align}
The next proposition lists important properties of this gap. 

\begin{proposition}\label{prop:m}
	Let $i,j\in [K]$. (i) It holds $m(i,j)<+\infty$. (ii) It holds $m(i,j)=d(\bs{\mu}_j-\bs{\mu}_i,(\Int(C))^c\cap (\bs{\mu}_j-\bs{\mu}_i-C))$. (iii) We have $m(i,j)>0$ iff $i\prec_C j$. (iv) It holds $m(i,j)=\min_{n\in [N]}(w_n^{\mathsf{T}}(\bs{\mu}_j-\bs{\mu}_i))^+/\alpha_n$.
\end{proposition}


%

Similarly, given $i,j\in[K]$, we define $M(i,j)$ as the minimum increment in $\bs{\mu}_j$ in an arbitrary direction of increase that makes design $i$ weakly dominated by design $j$, that is,
\begin{align}\label{eq:M}
	M(i,j):=\inf\{s\geq 0\mid \exists& \bs{u}\in B(\bs{0},1)\cap C\colon \notag \\
	& \bs{\mu}_j+s\bs{u}\in \bs{\mu}_i+C\}~.
\end{align}
We state the fundamental properties of this gap next.

\begin{proposition}\label{prop:M}
	Let $i,j\in[K]$. (i) It holds $M(i,j)<+\infty$. (ii) It holds $M(i,j)=d(\bs{\mu}_j-\bs{\mu}_i,C\cap (\bs{\mu}_j-\bs{\mu}_i+C))$. (iii) We have $M(i,j)=0$ iff $i\preceq_C j$. 
\end{proposition}

\begin{remark}\label{rem:Cgap}
	From \eqref{eq:M}, it is clear that using a larger cone would result in smaller values of $M(i,j)$.
\end{remark}

As an immediate consequence of Propositions \ref{prop:m}(ii, iii) and \ref{prop:M}(ii,iii), we obtain the following corollary.

\begin{corollary}\label{cor:mM}
	Let $i,j\in[K]$. Then, exactly one of the following cases holds.\\
	(i) We have both $i\prec_C j$ and $i\preceq_C j$ iff $\bs{\mu}_j-\bs{\mu}_i\in\Int(C)$ iff $m(i,j)>0$ and $M(i,j)=0$.\\
	(ii) We have both $i\nprec_C j$ and $i\preceq_C j$ iff $\bs{\mu}_j-\bs{\mu}_i\in\bd(C)$ iff $m(i,j)=M(i,j)=0$.\\
	(iii) We have both $i\nprec_C j$ and $i\npreceq_C j$ iff $\bs{\mu}_j-\bs{\mu}_i \in C^c$ iff $m(i,j)=0$ and $M(i,j)>0$.\\
	In particular, we have $m(i,j)=0$ or $M(i,j)=0$.
\end{corollary}



For each $i \in [K]$, let $\Delta^\ast_i := \max_{j \in P^\ast} m(i,j)$. By Proposition \ref{prop:M}(iii), it is clear that $\Delta^\ast_i = 0$ if $i\in P^\ast$. We consider an $(\epsilon,\delta)$-PAC (probably approximately correct) Pareto set identification setup under which the estimated Pareto set $P\subseteq[K]$ returned by the learner needs to satisfy the condition in the next definition. 
\begin{definition}\label{defn:PAC}(Success condition)
	Let $\epsilon>0$, $\delta \in (0,1)$. A set $P\subseteq[K]$ is called an \emph{$(\epsilon,\delta)$-PAC Pareto set} if the following properties hold with probability at least $1-\delta$: ($i$) $\cup_{i \in P} (\bs{\mu}_i + (B(\bs{0}, \epsilon)\cap C) - C) \supseteq \cup_{i \in P^*} (\bs{\mu}_i - C)$; ($ii$) for every $i \in P \setminus P^\ast$, it holds $\Delta^*_i \leq \epsilon$.
\end{definition}

\begin{remark}\label{rem:PAC}
	Note that property $(i)$ in Definition~\ref{defn:PAC} is equivalent to $(i')$ $\cup_{i\in P}(\bs{\mu}_i+(B(0,\bs{\epsilon})\cap C)-C)\supseteq \{\bs{\mu}_i\mid i\in P^\ast\}$. The implication $(i)\Rightarrow(i')$ is obvious. To see $(i')\Rightarrow(i)$, suppose that $(i')$ holds and let $\bs{v}\in \bs{\mu}_i-C$ for some $i\in P^\ast$. Then, by $(i')$, $\bs{\mu}_i\in \bs{\mu}_j+(B(0,\bs{\epsilon})\cap C)-C$ for some $j\in P$. Moreover, $C+C=C$ since $C$ is a convex cone. Hence, $\bs{v}\in \bs{\mu}_i-C\subseteq \bs{\mu}_j+(B(0,\bs{\epsilon})\cap C)-C-C=\bs{\mu}_j+(B(0,\bs{\epsilon})\cap C)-C$, showing $(i)$. A similar covering property is used in \citet[Definition 3.5]{umar} in the context of deterministic convex vector optimization.
\end{remark}

Property $(i)$ in Definition \ref{defn:PAC} is an $\epsilon$-covering requirement for the Pareto set $P^\ast$; in view of Remark \ref{rem:PAC}, it is equivalent to the following: for every $i\in P^\ast$, there exist $j\in P$ and $\bs{u}\in B(\bs{0},1)\cap C$ such that $\bs{\mu}_i\preceq_C \bs{\mu}_j+\epsilon\bs{u}$. Roughly speaking, although $P$ might not contain all Pareto optimal designs, it is required to include a close-enough design for each Pareto optimal design. Moreover, although some designs in $P$ might be suboptimal, property $(ii)$ in Definition \ref{defn:PAC} bounds the gaps of such designs; hence, it controls the quality of all returned designs.

\begin{remark}\label{rem:auersuccess}
	Ignoring the differences in gap definitions, let us compare our success condition with the two alternative success conditions (SC1 and SC2) in \citet{auer2016pareto} at a structural level. All three conditions require small suboptimality gap: $\Delta_i^\ast\leq \epsilon$. While SC1 requires $P^*$ to be a subset of the returned set, SC2 and ours impose the weaker $\epsilon$-covering requirement, which is in line with \citet{even2006action,zuluaga2016varepsilon}. Finally, SC2 has a sparsity requirement for the returned set of designs, which we do not have in Definition \ref{defn:PAC}. For instance, when $D=1$, our success condition is satisfied when all returned designs are $\epsilon$-optimal, which might still violate SC2 depending on the configuration of the reward vectors. However, due to the different definitions of the gaps, our success condition is mathematically incomparable with SC1, SC2. 
\end{remark}

There is an innate connection between gaps defined in this section and the ordering complexity terms defined in Section~\ref{sec:prelim}. The following remark summarizes this relation, which plays an important role in characterizing the sample complexity.

\begin{remark} \label{rem:intuitive}
	({Intuitive explanations of $\beta_1$, $\beta_2$}) Let $\bs{\Delta}_{ij} = \bs{\mu}_j - \bs{\mu}_i$. Consider $i$, $j$ such that $\bs{\Delta}_{ij} \notin C$, i.e., $\bs{\mu}_i \npreceq_C \bs{\mu}_j$. By definition, $M(i,j) = d (\bs{\Delta}_{ij} , C \cap (\bs{\Delta}_{ij} + C) )$ (numerator in \eqref{eq:Csup1}) is the minimum change in $\bs{\Delta}_{ij}$ in an arbitrary direction of increase that makes $i$ weakly dominated by $j$, while $d(\bs{\Delta}_{ij}, C)$ (denominator in \eqref{eq:Csup1}) is the minimum change in $\bs{\Delta}_{ij}$ that makes $i$ weakly dominated by $j$. In scalar and multi-objective problems, we have $M(i,j) = d(\bs{\Delta}_{ij}, C)$. However, the ratio $M(i,j)/d(\bs{\Delta}_{ij}, C)$ can be as large as $\beta_1$ in vector optimization. This makes learning fundamentally more difficult in vector optimization when $\beta_1 > 1$. In the scalar or multi-objective case, problems with larger gaps $M(i,j)$ are easier (success conditions can be met with fewer samples). However, in general vector optimization, larger gaps do not imply easier problems, what is important is the ratio $M(i,j)/d(\bs{\Delta}_{ij}, C)$. A similar argument can be made for $\beta_2$.
\end{remark}

\section{LOWER BOUNDS}\label{sec:lowerbounds}

We provide gap-dependent and worst-case sample complexity bounds for any algorithm that satisfies the success condition in Definition \ref{defn:PAC}. For each $i\in P^\ast$, we define $\Delta^+_i:=\min_{j\in P^\ast\setminus\{i\}}M(i,j)$ and $\tilde{\Delta}^{\epsilon}_i:=\max\{\Delta_i^+,\epsilon\}$. For each $i\notin P^\ast$, we have $\Delta_i^\ast=\max_{j\in P^\ast}m(i,j)$ and define $\tilde{\Delta}_i^\epsilon:=\max\{\Delta_i^\ast,\epsilon\}$. Finally, let $\mathcal{A}$ be the class of all algorithms that return, for any given $\epsilon>0$ and $\delta \in (0,1)$, a set $P\subseteq[K]$ satisfying Definition~\ref{defn:PAC} for every given distribution of $\bs{X}_t$ with $\bs{X}_t \in [0,1]^D$ almost surely for each $t\in [T]$. 

\begin{theorem}\label{lowerboundthm} (Gap-dependent lower bound)
	Suppose that $\bs{\mu}_i \in [\frac14,\frac34]^D$ for each $i\in [K]$ and $\epsilon \leq 1/8$. Then, there exist such distributions with the property that every algorithm in $\mathcal{A}$ requires $\Omega(\sum_{i\in [K]}\frac{1}{(\tilde{\Delta}^\epsilon_i)^2}\log(\frac{1}{\delta}))$ samples to work correctly. 
\end{theorem}

\begin{remark}\label{rem:auerlb}
	(a) Theorem \ref{lowerboundthm} is analogous to Theorem 17 in \cite{auer2016pareto}. However, this result works under their SC1, which is structurally much stricter than and mathematically incomparable with our success condition in Definition \ref{defn:PAC}; see Remark \ref{rem:auersuccess}. Ours is closer to their SC2, for which there is no lower bound analysis in \cite{auer2016pareto}. Hence, by relaxing the structure of their SC2, we are able to provide lower bounds on sample complexity for a success condition with an $\epsilon$-covering requirement, as opposed to the full coverage requirement in their SC1.\\
	(b) Different from their proof in the multi-objective setting, our proof constructs a special direction vector $\bs{z}^\ast\in \text{int}(C)$ along which the rewards and their (empirical) means for a fixed design are ordered linearly. Precisely, $\bs{z}^\ast$ is chosen in such a way that $\bs{z}^\ast \in \bs{u}+C$ for every $\bs{u}\in B(\bs{0},1)\cap C$. Using this vector, we analyze the following four cases: ($i$) When $i\in P^\ast\setminus P$, we modify the mean reward of $i$ as $\bs{\mu}_i^\prime:=\bs{\mu}_i+2\tilde{\Delta}^\epsilon_i \bs{z}^\ast$ and show that $i$ is not covered (in the sense of Definition \ref{defn:PAC}) by any $j\in [K]$ that covers $i$ in the original case. Hence, $i$ must be returned by the algorithm in the modified case. ($ii$) When $i\in P^\ast\cap P$, we have $\Delta_i^+=M(i,j)$ for some $j\in P^\ast\setminus \{i\}$. We take $\bs{\mu}_i^\prime:=\bs{\mu}_i-3k^\ast\tilde{\Delta}^\epsilon_i \bs{z}^\ast$ with $k^\ast:=\max_{n\in[N]}\frac{\alpha_n}{\bs{w}_n^{\mathsf{T}}\bs{z}^\ast}$ and show that $m(i,j)\geq 2\epsilon$ in the modified case. Hence, $i$ cannot be returned by the algorithm. ($iii$) The case $i\in P\setminus P^\ast$ is similar to ($ii$). ($iv$) When $i\notin P^\ast\cup P$, we take $\bs{\mu}^\prime_i := \bs{\mu}_i +(1+2k^\ast) \tilde{\Delta}^\epsilon_i\bs{z}^\ast$ and show that $i$ must be returned. Since we are able to change the returned set in each possibility, the rest of the proof follows as in \cite{auer2016pareto}. The main technical novelties in the proof are twofold: 1) The direction vector $\bs{z}^\ast$ is constructed by using the structure of the ordering cone. 2) Compared to \cite{auer2016pareto}, case ($i$) is new to our setting as $P^\ast\subseteq P$ is required there, which is relaxed by our Definition \ref{defn:PAC}.
\end{remark}

\begin{theorem}\label{lowerboundthm_new} 
	(Worst-case lower bound)
	Suppose that there are $K \geq 4$ designs. Given $\epsilon >0$, there exist mean reward vectors and norm-subgaussian noise distributions under which any algorithm in $\mathcal{A}$ requires $\Omega( (K \beta_2^2/\epsilon^2) \log(1/\delta) )$ samples to work correctly.
\end{theorem}

The proof of Theorem \ref{lowerboundthm_new} requires utilizing the geometry of the ordering cone to construct a special worst-case problem instance that is hard to distinguish from other similar problem instances. In this special instance, while mean reward vectors of most designs are identical and Pareto optimal, and there is another Pareto optimal design that is in a specially perturbed direction. 
We show that, when the noise is in a special direction determined by $\beta_2$, a change in its distribution that is proportional to $\epsilon/\beta_2$ will result in a change in the returned set of designs. 
This lower bound manifests a distinct feature of vector optimization which is not present in the scalar or multi-objective setting under noisy feedback, i.e., geometry of the ordering cone characterizes the hardness of the Pareto set identification problem. Note that we indeed have $\beta=\beta_2$ for all ordering cones that are at least as large as $\R^D_+$ (Theorem \ref{thm:beta}(ii)) and all ordering cones in the biobjective case (Example \ref{ex:2Dsmall}).

\section{NA\"{I}VE ELIMINATION AND ITS SAMPLE COMPLEXITY}\label{sec:alg}

In this section, we introduce the na\"{i}ve elimination algorithm that is used for $(\epsilon,\delta)$-PAC Pareto set identification. The algorithm operates in the same fashion as the na\"{i}ve elimination algorithm used for $(\epsilon,\delta)$-PAC best arm identification \citep{even2006action}. For $D=1$, it is known that na\"{i}ve elimination is a simple yet provable and practical algorithm that is known to work reasonably well when $K$ is not very large (see the discussion in \citet{hassidim2020optimal}).\footnote{Our aim is to show that even such a simple algorithm can be successful when its sampling budget is tuned according to ordering complexity. We think that many of the successive elimination and adaptive sampling techniques developed for the scalar case can be adapted to devise sample-efficient algorithms for vector optimization. We leave this interesting direction as future work.} 

\begin{algorithm}[h!]
	\caption{Na\"{i}ve Elimination} \label{alg:algo}
	\begin{algorithmic}
		\STATE \textbf{Inputs}: $\epsilon$, $\delta$, $[K]$, $C$, $\beta$, $L$
		\STATE \textbf{Initialize}: $\hat{\bs{\mu}}_i = N_i = 0$, $i \in [K]$
		\FOR{$t=1,2,\ldots,LK$}
		\STATE Evaluate $I_t = (t-1 \mod K) + 1$
		\STATE Observe reward $\bs{X}_t$
		\STATE $N_{I_t} = N_{I_t} +1$;
		$\hat{\bs{\mu}}_{I_t} = ( (N_{I_t}-1) \hat{\bs{\mu}}_{I_t} + \bs{X}_t) / N_{I_t}$
		\ENDFOR
		\STATE \textbf{Return}: $P:=\{i\in[K]\mid \nexists j\in [K]\colon i\widehat{\preceq}_{C\setminus\{\bs{0}\}} j\}$
	\end{algorithmic}
\end{algorithm}

Na\"{i}ve elimination takes as inputs $(\epsilon,\delta)$ and the polyhedral ordering cone $C$ defined by matrix $W$. It evaluates each design $L$ times to form empirical means $\hat{\bs{\mu}}_i$, $i \in [K]$, where $L\geq 1$ is a positive integer that is used as the exploration parameter and it is set depending on $\epsilon$, $\delta$, $K$ and the ordering complexity $\beta$ of $C$.\footnote{The computation of $\beta$, which is offline and done once at the beginning, is trivially simple when $C \supseteq \mathbb{R}^D_{+}$ (Theorem \ref{thm:beta}). For general cones, we show how $\beta$ can be computed in the supplemental document; however, we do not have a computationally efficient procedure for it as the problem is not convex. Nevertheless, as we mention in Remark~\ref{rem:betacalc}, problems in vector optimization generally have small objective dimension (not to be confused with design space dimension).} Hence, $T=LK$ evaluations are made in total. Then, the algorithm computes and returns a random set $P$ by using Definition \ref{defn:paretooptimal}. This is done by checking for each design $i$ whether there exists another design $j$ with $i \widehat{\preceq}_{C \setminus \{ \bs{0} \}} j$, where $\widehat{\preceq}_{C\setminus\{\bs{0}\}}$ is the random partial order that is defined by using $\hat{\bs{\mu}}_i$ in place of $\bs{\mu}_i$, $i \in [K]$, in the standard expression of $\preceq_{C \setminus \{ \bs{0} \}}$; see Definitions \ref{defn:cone}, \ref{defn:partialorder} (the random relations $\widehat{\prec}_C$ and $\widehat{\preceq}_C$ are defined similarly). Hence, we set $P:=\{i\in[K]\mid \nexists j\in [K]\colon i\widehat{\preceq}_{C\setminus\{\bs{0}\}} j\}$. One possible implementation of na\"{i}ve elimination is given in Algorithm \ref{alg:algo}.

For a given $(\epsilon,\delta)$, we define the per-design sampling budget as
$g(\epsilon,\delta) :=  \Big\lceil \frac{4 \beta^2 c^2 \sigma^2}{\epsilon^2} \log \Big(\frac{4 D}{\delta}\Big) \Big\rceil$.
The following theorem characterizes the sample complexity of na\"{i}ve elimination for $(\epsilon,\delta)$-PAC Pareto set identification.
\begin{theorem}\label{thm:samplecomp}
	When na\"{i}ve elimination is run with $L= g(\epsilon,2\delta/(K (K-1))$, the returned Pareto set $P$ is an $(\epsilon,\delta)$-PAC Pareto set. 
\end{theorem}
Theorem \ref{thm:samplecomp} gives an upper bound on the sample complexity that scales as $\tilde{O}(K \beta^2/\epsilon^2)$ which nearly matches the lower bound in Theorem~\ref{lowerboundthm_new}.\footnote{Indeed, the sample complexity is $O( K \beta^2 \log(K D / \delta) / \epsilon^2)$.} It can also be shown that, when $D=1$, the sample complexity of na\"{i}ve elimination matches the one given in \citet{even2006action}.

{Let us provide a sketch of the proof of Theorem \ref{thm:samplecomp}. Given $i,j\in[K]$, let us introduce empirical estimates for $m(i,j)$ and $M(i,j)$, defined as $\hat{m}(i,j) := d(\hat{\bs{\mu}}_j - \hat{\bs{\mu}}_i, (\Int(C))^c\cap (\hat{\bs{\mu}}_j-\hat{\bs{\mu}}_i-C))$ and $\hat{M}(i,j) := d(\hat{\bs{\mu}}_j - \hat{\bs{\mu}}_i, C\cap (\hat{\bs{\mu}}_j-\hat{\bs{\mu}}_i+C))$, respectively. As a first step, we show that $M(i,j)> \epsilon$ implies $\hat{M}(i,j)>0$ whenever $i\in P^\ast$, $j\in [K] \setminus \{ i \}$, as well as that $m(i,j) > \epsilon$ implies $\hat{m}(i,j) >0$ whenever $i \notin P^\ast$, $j \in P^\ast$. In the second step, we show that these implications are sufficient to conclude that $P$ is an $(\epsilon,\delta)$-PAC Pareto set. Then, for $i,j\in [K]$ with $i\neq j$, we define $\bs{\Delta}_{ij} = \bs{\mu}_j - \bs{\mu}_i$, $\hat{\bs{\Delta}}_{ij} = \hat{\bs{\mu}}_j - \hat{\bs{\mu}}_i$, and show that the following conditions in terms of the deviation of $\hat{\bs{\Delta}}_{ij}$ from $\bs{\Delta}_{ij}$ are sufficient for $(\epsilon,\delta)$-PAC Pareto set identification: \\
	\textbf{Condition a.} \emph{For every $i\in P^\ast$ and $j\in [K]\setminus\{i\}$, $d(\bs{\Delta}_{ij},C\cap (\bs{\Delta}_{ij}+C))>\epsilon$ implies that $\|\hat{\bs{\Delta}}_{ij}-\bs{\Delta}_{ij}\|_2< d(\bs{\Delta}_{ij},C)$.}\\
	\textbf{Condition b.} \emph{For every $i\notin P^\ast$ and $j\in P^\ast$, $d(\bs{\Delta}_{ij},(\Int(C))^c\cap (\bs{\Delta}_{ij}-C))>\epsilon$ implies that $\|\hat{\bs{\Delta}}_{ij}-\bs{\Delta}_{ij}\|_2< d(\bs{\Delta}_{ij},(\Int(C))^c)$.}
}\\
Let us also define $\theta_{ij}:=\frac{d(\bs{\Delta}_{ij},C)}{d(\bs{\Delta}_{ij},C\cap(\bs{\Delta}_{ij}+C))}$ if $\bs{\Delta}_{ij}\notin C$, $\theta_{ij}:= \frac{d(\bs{\Delta}_{ij},(\Int(C))^c)}{d(\bs{\Delta}_{ij},(\Int(C))^c \cap(\bs{\Delta}_{ij}-C))}$ if $\bs{\Delta}_{ij}\in \Int(C)$, and $\theta_{ij}:=1$ if $\bs{\Delta}_{ij}\in\bd(C)$. In the third step, using properties of norm-subgaussian noise vectors, we show the existence of a constant $c>0$ (free of all problem parameters) such that, for all $i\neq j$, $||   \hat{\bs{\Delta}}_{ij} - \bs{\Delta}_{ij}   ||_2 \leq \epsilon \theta_{ij}$ with probability at least $1-\delta$ provided that na\"{i}ve elimination is run with $L= g(\epsilon,2\delta/(K (K-1))$. Finally, we complete the proof by checking that Conditions a and b hold with probability at least $1-\delta$.


\section{NUMERICAL RESULTS}\label{sec:experiment}

We use SNW dataset from \citet{zuluaga2016varepsilon}. It consists of $206$ different hardware implementations of a sorting network. The objectives are the area and throughput of the network when synthesized on an FPGA ($D=2$). Since we consider maximization problems, we use the negative of area as objective value. The mean rewards of designs are taken as the objective values in the dataset. The reward vector of a design is formed by adding independent zero mean Gaussian noise with variance $\sigma^2 = 1$ to the mean value of each objective of the design. 

We consider the polyhedral cone $C_\theta$ in Example \ref{ex:2Dsmall}, which is parametrized by an angle $\theta \in (0, \pi)$.
We use $C_{\pi/4}$, $C_{\pi/2}$ and $C_{3\pi/4}$ in our simulations (Figure \ref{fig:cones}). $P^*_{\theta}$ represents the true Pareto optimal set under the ordering induced by $C_{\theta}$. The Pareto sets for $C_{\pi/4}$, $C_{\pi/2}$ and $C_{3\pi/4}$ are shown in Figure \ref{fig:paretosets}. Some statistics of the gaps $\Delta^*_i$ of designs $i \in [K]  \setminus P^*_{\theta}$ are given in Table \ref{tbl:gapsofsuboptimals}.

\begin{table}
	\centering
	\caption{Gap Statistics of Designs $i \in [K]  \setminus P^*_{\theta}$.}
	\label{tbl:gapsofsuboptimals}
	\begin{tabular}{lrrr}
		\toprule
		{} & $\Delta^*_{i}$ ($C_{\pi/4}$) &  $\Delta^*_{i}$ ($C_{\pi/2}$)  &  $\Delta^*_{i}$ ($C_{3\pi/4}$) \\
		\midrule
		count &  153  &   180 &    196 \\
		mean  &   0.666 &     0.906 &      1.343 \\
		std   &     0.520 &     0.711 &      0.869 \\
		min   &    0.001 &     0.004 &      0.018 \\
		max   &     2.576 &     3.544 &      4.726 \\
		\bottomrule
	\end{tabular}
\end{table}

\begin{table}
	\centering
	\caption{Success Rate (\%) of Na\"{i}ve Elimination.}
	\label{tbl:ressults}
	\begin{tabular}{llrrr}
		\toprule
		$L$ & $\epsilon$  &  $C_{\pi/4}$ &  $C_{\pi/2}$ &  $C_{3\pi/4}$   \\
		\midrule
		\multirow{3}{*}{$10^3$} 
		& $10^{-2}$ &   0 &   0 &   29 \\
		& $10^{-1}$ &  78 &  99 &  100 \\
		\cline{1-5} \\ [-1em]
		\multirow{3}{*}{$10^4$} 
		& $10^{-2}$ &  22 &  24 &   85 \\
		& $10^{-1}$ & 100 & 100 &  100 \\
		\cline{1-5} \\ [-1em]
		\multirow{3}{*}{$10^5$} 
		& $10^{-2}$ & 100 &  99 &  100 \\
		& $10^{-1}$ & 100 & 100 &  100 \\
		\bottomrule
	\end{tabular}
\end{table}

The simulation code is available in the supplemental material. As there is no work that considers vector optimization with stochastic bandit feedback, we do not compare with any other method. Instead, we illustrate how the performance of na\"{i}ve elimination varies as a function of the number of samples from each design and shape of the ordering cone. We set $\delta = 0.01$ in all simulations. The reported results correspond to the average of $100$ independent runs. As typically observed in the best arm identification literature, the theoretical value of $L$ is very large. For instance, for $\theta = \pi/2$, if $\epsilon=0.1$, then $L \approx 38.8 \times 10^3$; if $\epsilon=0.01$, then $L \approx 38.8 \times 10^5$. Therefore, instead of using theoretical values, we evaluate the results for different $L$ and $\epsilon$ values. The results are provided in Table \ref{tbl:ressults}. For fixed $L$ and $\epsilon$, the success rate increases as $\theta$ increases. For fixed $\theta$ and $L$, the success rate increases as $\epsilon$ increases because both success conditions become easier to satisfy. For a fixed $\theta$ and $\epsilon$, the success rate increases as $L$ increases since more samples mean less noise in the estimates. For some $(\epsilon,\theta)$ pairs, all runs are successful even when $L$ is much smaller than its theoretical value.

\section{CONCLUSIONS}\label{sec:conc}

We propose vector optimization problems with stochastic bandit feedback. We identify fundamental cone-dependent gaps that characterize the learning difficulty of the Pareto set. We derive sample complexity bounds for $(\epsilon,\delta)$-PAC Pareto set identification. Our introduction of noisy bandit feedback to vector optimization brings forth many interesting future research directions. In particular, the design of sample-efficient adaptive algorithms remains as an open problem that we seek to address in the future. For instance, when learning the Pareto set in large design spaces, one can incorporate powerful surrogate models such as Gaussian processes. It should also be possible to extend entropy search \citep{hernandez2016predictive} and hypervolume-based methods \citep{shah2016pareto} proposed for multi-objective Pareto set identification problems to vector optimization.

\subsubsection*{Acknowledgements}

This work is supported by the Scientific and Technological Research Council of Turkey (T\"{U}B{\.I}TAK) under Grant 121E159.

\newpage

\bibliographystyle{abbrvnat}
\bibliography{references.bib}

\begin{thebibliography}{35}
\providecommand{\natexlab}[1]{#1}
\providecommand{\url}[1]{\texttt{#1}}
\expandafter\ifx\csname urlstyle\endcsname\relax
  \providecommand{\doi}[1]{doi: #1}\else
  \providecommand{\doi}{doi: \begingroup \urlstyle{rm}\Url}\fi

\bibitem[Ararat et~al.(2022)Ararat, Ulus, and Umer]{umar}
{\c{C}}.~Ararat, F.~Ulus, and M.~Umer.
\newblock A norm minimization-based convex vector optimization algorithm.
\newblock \emph{Journal of Optimization Theory and Applications}, 192:\penalty0
  681--712, 2022.

\bibitem[Audibert et~al.(2010)Audibert, Bubeck, and Munos]{audibert2010best}
J.-Y. Audibert, S.~Bubeck, and R.~Munos.
\newblock Best arm identification in multi-armed bandits.
\newblock In \emph{Proc. Conference on Learning Theory}, pages 41--53, 2010.

\bibitem[Auer et~al.(2016)Auer, Chiang, Ortner, and Drugan]{auer2016pareto}
P.~Auer, C.-K. Chiang, R.~Ortner, and M.~Drugan.
\newblock Pareto front identification from stochastic bandit feedback.
\newblock In \emph{Proc. 19th Intertnational Conference on Artificial
  Intelligence and Statistics}, pages 939--947, 2016.

\bibitem[Bubeck et~al.(2009)Bubeck, Munos, and Stoltz]{bubeck2009pure}
S.~Bubeck, R.~Munos, and G.~Stoltz.
\newblock Pure exploration in multi-armed bandits problems.
\newblock In \emph{Proc. International Conference on Algorithmic Learning
  Theory}, pages 23--37, 2009.

\bibitem[Degenne et~al.(2019)Degenne, Koolen, and M{\'e}nard]{degenne2019non}
R.~Degenne, W.~M. Koolen, and P.~M{\'e}nard.
\newblock Non-asymptotic pure exploration by solving games.
\newblock \emph{Advances in Neural Information Processing Systems}, 32, 2019.

\bibitem[Degenne et~al.(2020)Degenne, M{\'e}nard, Shang, and
  Valko]{degenne2020gamification}
R.~Degenne, P.~M{\'e}nard, X.~Shang, and M.~Valko.
\newblock Gamification of pure exploration for linear bandits.
\newblock In \emph{Proc. International Conference on Machine Learning}, pages
  2432--2442, 2020.

\bibitem[Drugan and Nowe(2013)]{drugan2013designing}
M.~M. Drugan and A.~Nowe.
\newblock Designing multi-objective multi-armed bandits algorithms: A study.
\newblock In \emph{The 2013 International Joint Conference on Neural Networks
  (IJCNN)}, pages 1--8, 2013.

\bibitem[Even-Dar et~al.(2006)Even-Dar, Mannor, Mansour, and
  Mahadevan]{even2006action}
E.~Even-Dar, S.~Mannor, Y.~Mansour, and S.~Mahadevan.
\newblock Action elimination and stopping conditions for the multi-armed bandit
  and reinforcement learning problems.
\newblock \emph{Journal of Machine Learning Research}, 7\penalty0 (6), 2006.

\bibitem[Garivier and Kaufmann(2016)]{garivier2016optimal}
A.~Garivier and E.~Kaufmann.
\newblock Optimal best arm identification with fixed confidence.
\newblock In \emph{Proc. Conference on Learning Theory}, pages 998--1027, 2016.

\bibitem[Garivier and Kaufmann(2021)]{garivier2021nonasymptotic}
A.~Garivier and E.~Kaufmann.
\newblock Nonasymptotic sequential tests for overlapping hypotheses applied to
  near-optimal arm identification in bandit models.
\newblock \emph{Sequential Analysis}, 40\penalty0 (1):\penalty0 61--96, 2021.

\bibitem[Hassidim et~al.(2020)Hassidim, Kupfer, and
  Singer]{hassidim2020optimal}
A.~Hassidim, R.~Kupfer, and Y.~Singer.
\newblock An optimal elimination algorithm for learning a best arm.
\newblock \emph{Advances in Neural Information Processing Systems},
  33:\penalty0 10788--10798, 2020.

\bibitem[Hayes et~al.(2022)Hayes, R{\u{a}}dulescu, Bargiacchi,
  K{\"a}llstr{\"o}m, Macfarlane, Reymond, Verstraeten, Zintgraf, Dazeley,
  Heintz, et~al.]{hayes2022practical}
C.~F. Hayes, R.~R{\u{a}}dulescu, E.~Bargiacchi, J.~K{\"a}llstr{\"o}m,
  M.~Macfarlane, M.~Reymond, T.~Verstraeten, L.~M. Zintgraf, R.~Dazeley,
  F.~Heintz, et~al.
\newblock A practical guide to multi-objective reinforcement learning and
  planning.
\newblock \emph{Autonomous Agents and Multi-Agent Systems}, 36\penalty0
  (1):\penalty0 1--59, 2022.

\bibitem[Hern{\'a}ndez-Lobato et~al.(2016)Hern{\'a}ndez-Lobato,
  Hernandez-Lobato, Shah, and Adams]{hernandez2016predictive}
D.~Hern{\'a}ndez-Lobato, J.~Hernandez-Lobato, A.~Shah, and R.~Adams.
\newblock Predictive entropy search for multi-objective {Bayesian}
  optimization.
\newblock In \emph{Proc. International Conference on Machine Learning}, pages
  1492--1501, 2016.

\bibitem[Heyde and L\"{o}hne(2011)]{heydelohne}
F.~Heyde and A.~L\"{o}hne.
\newblock Solution concepts in vector optimization: a fresh look at an old
  story.
\newblock \emph{Optimization}, 60:\penalty0 1421--1440, 2011.

\bibitem[Jahn(2011)]{jahn}
J.~Jahn.
\newblock \emph{Vector Optimization}.
\newblock Springer, 2nd edition, 2011.

\bibitem[Jayatunga et~al.(2022)Jayatunga, Xie, Ruder, Schulze, and
  Meier]{jayatunga2022ai}
M.~K. Jayatunga, W.~Xie, L.~Ruder, U.~Schulze, and C.~Meier.
\newblock {AI} in small-molecule drug discovery: A coming wave.
\newblock \emph{Nat. Rev. Drug Discov}, 21:\penalty0 175--176, 2022.

\bibitem[Jin et~al.(2019)Jin, Netrapalli, Ge, Kakade, and Jordan]{jin2019short}
C.~Jin, P.~Netrapalli, R.~Ge, S.~M. Kakade, and M.~I. Jordan.
\newblock A short note on concentration inequalities for random vectors with
  subgaussian norm.
\newblock \emph{arXiv preprint arXiv:1902.03736}, 2019.

\bibitem[Kabanov(1999)]{kabanov}
Y.~M. Kabanov.
\newblock Hedging and liquidation under transaction costs in currency markets.
\newblock \emph{Finance and Stochastics}, 3:\penalty0 237--248, 1999.

\bibitem[Kalyanakrishnan et~al.(2012)Kalyanakrishnan, Tewari, Auer, and
  Stone]{kalyanakrishnan2012pac}
S.~Kalyanakrishnan, A.~Tewari, P.~Auer, and P.~Stone.
\newblock {PAC} subset selection in stochastic multi-armed bandits.
\newblock In \emph{Proc. International Conference on Machine Learning},
  volume~12, pages 655--662, 2012.

\bibitem[Karnin et~al.(2013)Karnin, Koren, and Somekh]{karnin2013almost}
Z.~Karnin, T.~Koren, and O.~Somekh.
\newblock Almost optimal exploration in multi-armed bandits.
\newblock In \emph{Proc. International Conference on Machine Learning}, pages
  1238--1246, 2013.

\bibitem[Katz-Samuels and Scott(2018)]{katz2018feasible}
J.~Katz-Samuels and C.~Scott.
\newblock Feasible arm identification.
\newblock In \emph{Proc. International Conference on Machine Learning}, pages
  2535--2543, 2018.

\bibitem[Kaufmann and Kalyanakrishnan(2013)]{kaufmann2013information}
E.~Kaufmann and S.~Kalyanakrishnan.
\newblock Information complexity in bandit subset selection.
\newblock In \emph{Proc. Conference on Learning Theory}, pages 228--251, 2013.

\bibitem[Kov{\'a}{\v{c}}ov{\'a} and Rudloff(2021)]{kovavcova2021time}
G.~Kov{\'a}{\v{c}}ov{\'a} and B.~Rudloff.
\newblock Time consistency of the mean-risk problem.
\newblock \emph{Operations Research}, 69\penalty0 (4):\penalty0 1100--1117,
  2021.

\bibitem[Laskey et~al.(2015)Laskey, Mahler, McCarthy, Pokorny, Patil, Van
  Den~Berg, Kragic, Abbeel, and Goldberg]{laskey2015multi}
M.~Laskey, J.~Mahler, Z.~McCarthy, F.~T. Pokorny, S.~Patil, J.~Van Den~Berg,
  D.~Kragic, P.~Abbeel, and K.~Goldberg.
\newblock Multi-armed bandit models for 2d grasp planning with uncertainty.
\newblock In \emph{Proc. 2015 IEEE International Conference on Automation
  Science and Engineering (CASE)}, pages 572--579, 2015.

\bibitem[Lizotte and Laber(2016)]{lizotte2016multi}
D.~J. Lizotte and E.~B. Laber.
\newblock Multi-objective {Markov} decision processes for data-driven decision
  support.
\newblock \emph{The Journal of Machine Learning Research}, 17\penalty0
  (1):\penalty0 7378--7405, 2016.

\bibitem[L\"{o}hne(2011)]{lohne}
A.~L\"{o}hne.
\newblock \emph{Vector Optimization with Infimum and Supremum}.
\newblock Springer, 2011.

\bibitem[Qin et~al.(2017)Qin, Klabjan, and Russo]{qin2017improving}
C.~Qin, D.~Klabjan, and D.~Russo.
\newblock Improving the expected improvement algorithm.
\newblock \emph{Advances in Neural Information Processing Systems}, 30, 2017.

\bibitem[Russo(2020)]{russo2020simple}
D.~Russo.
\newblock Simple {Bayesian} algorithms for best-arm identification.
\newblock \emph{Operations Research}, 68\penalty0 (6):\penalty0 1625--1647,
  2020.

\bibitem[Shah and Ghahramani(2016)]{shah2016pareto}
A.~Shah and Z.~Ghahramani.
\newblock Pareto frontier learning with expensive correlated objectives.
\newblock In \emph{Proc. International Conference on Machine Learning}, pages
  1919--1927, 2016.

\bibitem[Shang et~al.(2020)Shang, Heide, Menard, Kaufmann, and
  Valko]{shang2020fixed}
X.~Shang, R.~Heide, P.~Menard, E.~Kaufmann, and M.~Valko.
\newblock Fixed-confidence guarantees for {Bayesian} best-arm identification.
\newblock In \emph{Proc. International Conference on Artificial Intelligence
  and Statistics}, pages 1823--1832, 2020.

\bibitem[Teichert et~al.(2019)Teichert, Currie, K{\"u}fer, Miguel-Chumacero,
  S{\"u}ss, Walczak, and Currie]{teichert2019targeted}
K.~Teichert, G.~Currie, K.-H. K{\"u}fer, E.~Miguel-Chumacero, P.~S{\"u}ss,
  M.~Walczak, and S.~Currie.
\newblock Targeted multi-criteria optimisation in {IMRT} planning supplemented
  by knowledge based model creation.
\newblock \emph{Operations Research for Health Care}, 23:\penalty0 100185,
  2019.

\bibitem[Turgay et~al.(2018)Turgay, Oner, and Tekin]{turgay2018multi}
E.~Turgay, D.~Oner, and C.~Tekin.
\newblock Multi-objective contextual bandit problem with similarity
  information.
\newblock In \emph{Proc. International Conference on Artificial Intelligence
  and Statistics}, pages 1673--1681, 2018.

\bibitem[Van~Moffaert and Now{\'e}(2014)]{van2014multi}
K.~Van~Moffaert and A.~Now{\'e}.
\newblock Multi-objective reinforcement learning using sets of {Pareto}
  dominating policies.
\newblock \emph{The Journal of Machine Learning Research}, 15\penalty0
  (1):\penalty0 3483--3512, 2014.

\bibitem[Zhou et~al.(2014)Zhou, Chen, and Li]{zhou2014optimal}
Y.~Zhou, X.~Chen, and J.~Li.
\newblock Optimal {PAC} multiple arm identification with applications to
  crowdsourcing.
\newblock In \emph{Proc. International Conference on Machine Learning}, pages
  217--225, 2014.

\bibitem[Zuluaga et~al.(2016)Zuluaga, Krause, and
  P{\"u}schel]{zuluaga2016varepsilon}
M.~Zuluaga, A.~Krause, and M.~P{\"u}schel.
\newblock $\varepsilon$-pal: an active learning approach to the multi-objective
  optimization problem.
\newblock \emph{The Journal of Machine Learning Research}, 17\penalty0
  (1):\penalty0 3619--3650, 2016.

\end{thebibliography}


\begin{thebibliography}{1}
							\bibitem[Goldstein, 1977]{goldstein}
							A.~A. Goldstein. Optimization of {L}ipschitz continuous functions. \emph{Mathematical Programming}, 13: 14--22, 1977.
							\bibitem[Rockafellar, 1970]{rockafellar}
							R. T. Rockafellar. \emph{Convex Analysis}. Princeton University Press, 1970.
						\end{thebibliography}

\newpage

\appendix
\onecolumn

\section{TABLE OF NOTATIONS}

\begin{table*}[h!]
	\centering
	\fontsize{9}{9}\selectfont
		\begin{tabular}{ l l }
			\toprule
			\textbf{Notation} & \textbf{Description} \\ [.7ex]
			\midrule
			$D$ & Number of objectives  \\ [.7ex]
			$K$ & Number of designs  \\ [.7ex]
			$[K] = \{1,\ldots,K\}$ & Design space \\ [.7ex]
			$\bs{\mu}_i \in \mathbb{R}^D$ & Mean reward vector of design $i\in[K]$ \\ [.7ex]
			$C \subseteq \mathbb{R}^D$ & Ordering cone \\ [.7ex]
			$\Int(C)$ & Interior of $C$ \\ [.7ex]
			$C^{+} = \{ \bs{w} \in \mathbb{R}^D \mid \forall \bs{x} \in C : \bs{w}^{\mathsf{T}} \bs{x} \geq 0 \}$ & Dual cone of $C$ \\ [.7ex]
			$W = [\bs{w}_1, \ldots, \bs{w}_N]^{\mathsf{T}}$ & $N \times D$ parameter matrix of the polyhedral ordering cone $C$ \\ [.7ex]
			$\mathbb{R}^D_{+}$ & Positive orthant in $\mathbb{R}^D$  \\ [.7ex]
			$B(\bs{v},r)$ & Closed ball in $\mathbb{R}^D$ with center $\bs{v}\in\R^D$ and radius $r\geq 0$ \\ [.7ex]
			$\preceq_{C}$, $\prec_{C}$ & Partial order relations on $\R^D$ and on $[K]$ given in Definitions~\ref{defn:cone} and \ref{defn:partialorder} \\ [.7ex]
			$d(\bs{v},r) = \inf_{\bs{x} \in A} ||\bs{v} - \bs{x}||_2$ & Distance of vector $\bs{v} \in \mathbb{R}^D$ to set $A \subseteq \mathbb{R}^D$ \\ [.7ex]
			$\beta = \max\{ \beta_1, \beta_2 \}$ & Ordering complexity of cone $C$ ($\beta_1$, $\beta_2$ are defined in \eqref{eq:Csup1}, \eqref{eq:Csup2}, respectively.) \\ [.7ex]
			$\alpha_n = \sup_{\bs{u} \in B(\bs{0},1) \cap C} \bs{w}^\mathsf{T}_n \bs{u}$ & \\ [.7ex]
			$P^*$ & The (true) Pareto set according to cone $C$ \\ [.7ex]
			$m(i,j)$ & Minimum increment in $\bs{\mu}_i$ in an arbitrary direction of increase that makes design $i\in[K]$ \\
			& not strongly dominated by design $j\in[K]$ \\ [.7ex]
			$M(i,j)$ & Minimum increment in $\bs{\mu}_j$ in an arbitrary direction of increase that makes design $i\in[K]$  \\
			& weakly dominated by design $j\in[K]$ \\ [.7ex]
			$\Delta^*_i = \max_{j \in P^*} m(i,j)$ & \\ [.7ex]
			$\sigma$ & Subgaussian parameter of observation noise \\ [.7ex]
			$I_t$ & Design evaluated in round $t\in[T]$ in Algorithm~\ref{alg:algo} \\ [.7ex]
			$\hat{\bs{\mu}}_i $ & Empirical mean of the random reward of design $i\in[K]$ in Algorithm~\ref{alg:algo}\\ [.7ex]
			$P$ & Set of designs returned by Algorithm~\ref{alg:algo} \\ [.7ex]
			\bottomrule
		\end{tabular}
	\end{table*}


\section{SUPPLEMENTAL PROOFS} \label{sec:suppproof}

\subsection{Detailed version of Example \ref{ex:2Dsmall}} \label{sec:detailedexample}
\revc{We take $D=2$. Given $\bs{x}\in\R^2$, let $\alpha(\bs{x})\in [0,2\pi)$ denote the angle in the polar coordinates of $\bs{x}$. Let $\theta\in (0,\pi/2]$ and define the ordering cone $C_\theta:=\{\bs{x}\in\R^2 \mid \alpha(\bs{x})\in [\pi/4-\theta/2,\pi/4+\theta/2]\}$. Let $\bs{x}\notin C_\theta$. Using elementary planar geometry, it can be checked that $C_\theta\cap (\bs{x}+C_\theta)=C_\theta$ so that $d(\bs{x},C_\theta)=d(\bs{x},C_\theta\cap (\bs{x}+C_\theta))$ if $\alpha(\bs{x})\in [5\pi/4-\theta/2, 5\pi/4+\theta/2]$. If $\alpha(\bs{x})\in [0,\pi/4-\theta/2)\cup (\pi/4+\theta/2,3\pi/4+\theta/2]\cup [7\pi/4-\theta/2,2\pi)$, then $d(\bs{x},C_\theta\cap(\bs{x}+C_\theta))=d(\bs{x},C_\theta)\csc(\theta)$. If $\alpha(\bs{x})\in (3\pi/4+\theta/2,5\pi/4-\theta/2)$, then $d(\bs{x},C_\theta\cap(\bs{x}+C_\theta))=d(\bs{x},C_\theta)\csc(\theta)\sin(5\pi/4+\theta/2-\alpha(\bs{x}))\leq d(\bs{x},C_\theta)\csc(\theta)$. Similarly, if $\alpha(\bs{x})\in (5\pi/4+\theta/2,7\pi/4-\theta/2)$, then $d(\bs{x},C_\theta\cap(\bs{x}+C_\theta))=d(\bs{x},C_\theta)\csc(\theta)\sin(\alpha(\bs{x})-5\pi/4+\theta/2)\leq d(\bs{x},C_\theta)\csc(\theta)$. Hence, $\beta_1=\csc(\theta)$. By a similar calculation, it can be checked that $\beta_2=\csc(\theta)$ as well. If we take $\theta\in (\pi/2,\pi)$, then $\beta_1=\beta_2=1$ by Theorem \ref{thm:beta}(ii).} 

\subsection{Proof of Theorem \ref{thm:beta} and additional results for Section \ref{sec:prelim}}

\subsubsection{Computation of $\beta_1$ and $\beta_2$when the ordering cone is large}

Recall the definitions of $\beta_1,\beta_2$ in \eqref{eq:Csup2}, which depend on the geometry of the ordering cone $C$. We start by investigating the special case where the ordering cone is at least as large as the positive orthant \revc{(or a rotation of it)}.

\begin{proposition}\label{prop:conjCsup1}
	Suppose that $\bs{w}_n^{\mathsf{T}}\bs{w}_k\geq 0$ for each $n,k\in[N]$. Then, $d(\bs{x},C\cap (\bs{x}+C))=d(\bs{x},C)$ for every $\bs{x}\notin C$. In particular, $\beta_1=1$.
\end{proposition}

\begin{proof}
	Let us fix $\bs{x}\notin C$. Note that $C=\{\bs{y}\in\R^D\mid W\bs{y}\geq 0\}$ and $\bs{x}+C=\{\bs{y}\in\R^D\mid W\bs{y}\geq W\bs{x}\}$. Hence, $C\cap(\bs{x}+C)=\{\bs{y}\in\R^D\mid W\bs{y}\geq (W\bs{x})^+\}$. Since $\bs{x}\notin C$, there exists $n\in[N]$ such that $\bs{w}_n^{\mathsf{T}}\bs{x}<0$. Let $I(\bs{x}):=\{n\in[N]\mid \bs{w}_n^{\mathsf{T}}\bs{x}\leq 0\}\neq \emptyset$. We have $(\bs{w}_n^{\mathsf{T}}\bs{x})^+=0$ for each $n\in I(\bs{x})$ and $(\bs{w}_n^{\mathsf{T}}\bs{x})^+=\bs{w}_n^{\mathsf{T}}\bs{x}>0$ for each $n\in I(\bs{x})^c:=[N]\setminus I(\bs{x})$. Hence, $d(\bs{x},C\cap (\bs{x}+C))^2$ can be written as the optimal value of a quadratic optimization problem as follows:
	\begin{equation}\label{dopt1}
		d(\bs{x},C\cap(\bs{x}+C))^2=\inf\{\|\bs{y}-\bs{x}\|^2_2\mid \forall n\in I(\bs{x})\colon\bs{w}_n^{\mathsf{T}}\bs{y}\geq 0,\ \forall n\in I(\bs{x})^c\colon \bs{w}_n^{\mathsf{T}}\bs{y}\geq \bs{w}_n^{\mathsf{T}}\bs{x}\}~.
	\end{equation}
	As this is a convex optimization problem with affine constraints and finite optimal value, the standard Karush-Kuhn-Tucker conditions are necessary and sufficient for optimality. Hence, a vector $\bs{y}\in \R^D$ that satisfies the constraints of \eqref{dopt1} is optimal if and only if there exists a Lagrange multiplier vector $\bs{\lambda}\in \R^N_+$ such that the following conditions are satisfied:
	\begin{align}
		&2(\bs{y}-\bs{x})-\sum_{n\in [N]}\lambda_n\bs{w}_n =0~,\label{FOC}\\
		&\forall n\in I(\bs{x})\colon \lambda_n\bs{w}_n^{\mathsf{T}}\bs{y}=0~,\label{cs1}\\
		&\forall n\in I(\bs{x})^c\colon \lambda_n(\bs{w}_n^{\mathsf{T}}\bs{y}-\bs{w}_n^{\mathsf{T}}\bs{x}) =0~.\label{cs2}
	\end{align}
	Here, condition \eqref{FOC} is the first order condition for the Lagrangian of \eqref{dopt1} with respect to the primal variable $\bs{y}$; conditions \eqref{cs1} and \eqref{cs2} are the complementary slackness conditions.
	
	Let $\bs{y}\in\R^D$ be an optimal solution of \eqref{dopt1} with an associated Lagrange multiplier vector $\bs{\lambda}\in\R^N_+$. We claim that $\lambda_n=0$ for every $n\in I(\bs{x})^c$. To get a contradiction, suppose that $\lambda_{\bar{n}}>0$ for some $\bar{n}\in I(\bs{x})^c$. By \eqref{cs2}, we have $\bs{w}_{\bar{n}}^{\mathsf{T}}(\bs{y}-\bs{x})=\bs{w}_{\bar{n}}^{\mathsf{T}}\bs{y}-\bs{w}_{\bar{n}}^{\mathsf{T}}\bs{x}=0$. On the other hand, $\bs{y}-\bs{x}=\frac12 \sum_{n\in[N]}\lambda_n\bs{w}_n$ by \eqref{FOC}. Combining these, we get
	\begin{equation}\label{eq:claim}
		\sum_{n\in[N]}\lambda_{n}\bs{w}_{\bar{n}}^{\mathsf{T}}\bs{w}_n=0~.
	\end{equation}
	Note that $\bs{w}_n\in C^+$ for each $n\in[N]$. By the assumption on $\bs{w}_1,\ldots,\bs{w}_N$, \eqref{eq:claim} implies that $\lambda_n\bs{w}_{\bar{n}}^{\mathsf{T}}\bs{w}_n= 0$ for each $n\in [N]$. In particular, taking $n=\bar{n}$ gives $\lambda_{\bar{n}}\|\bs{w}_{\bar{n}}\|^2_2=\lambda_{\bar{n}}=0$, which is a contradiction. Therefore, the claim holds.
	
	Thanks to the above claim, the pair $(\bs{y},\bs{\lambda})$ satisfies the system
	\begin{align}
		&2(\bs{y}-\bs{x})-\sum_{n\in [N]}\lambda_n\bs{w}_n =0~,\label{FOCden}\\
		&\forall n\in [N]\colon \lambda_n\bs{w}_n^{\mathsf{T}}\bs{y}=0~.\label{csden}
	\end{align}
	Moreover, $\bs{w}_n^{\mathsf{T}}\bs{y}\geq 0$ for each $n\in [N]$ by the feasibility of $\bs{y}$ for \eqref{dopt1}. Hence, Karush-Kuhn-Tucker conditions are established for the quadratic optimization problem
	\begin{equation}\label{dopt2}
		d(\bs{x},C)^2=\inf\{\|\bs{y}-\bs{x}\|_2^2\mid \forall n\in [N]\colon \bs{w}_n^{\mathsf{T}}\bs{y}\geq 0\}
	\end{equation}
	and we conclude that $\bs{y}$ is optimal for \eqref{dopt2}. Therefore, $d(\bs{x},C)=d(\bs{x},C\cap(\bs{x}+C))=\|\bs{y}-\bs{x}\|_2$.
\end{proof}

\begin{proposition}\label{prop:conjCsup2}
	Suppose that $\bs{w}_n^{\mathsf{T}}\bs{w}_k\geq 0$ for each $n,k\in[N]$. Then, $d(\bs{x},(\Int(C))^c\cap (\bs{x}-C))=d(\bs{x},(\Int(C))^c)$ for every $\bs{x}\in \Int(C)$. In particular, $\beta_2=1$.
\end{proposition}

\begin{proof}
	Let us fix $\bs{x}\in\Int(C)$. We have $\bs{x}-C=\{\bs{y}\in\R^D\mid \forall n\in[N]\colon \bs{w}_n^{\mathsf{T}}\bs{x}\geq \bs{w}_n^{\mathsf{T}}\bs{y}\}$. Moreover, we have $(\Int(C))^c=\{\bs{y}\in\R^D\mid \exists k\in[N]\colon \bs{w}_k^{\mathsf{T}}\bs{y}\leq 0\}$. Hence,
	\[
	(\Int(C))^c\cap (\bs{x}-C)=\bigcup_{k\in[N]}\{\bs{y}\in\R^D\mid \forall n\in[N]\colon 
	\bs{w}_n^{\mathsf{T}}\bs{x}\geq \bs{w}_n^{\mathsf{T}}\bs{y},\ \bs{w}_k^{\mathsf{T}}\bs{y}\leq 0\}~,
	\]
	which is a finite union of convex polyhedra. In particular, we may write
	\begin{align}\label{eq:b2problemf}
		& d(\bs{x},(\Int(C))^c\cap (\bs{x}-C))^2 =\min_{k\in[N]}f_k(\bs{x})~, \\
		&\text{ where } f_k(\bs{x})\coloneqq \inf\{\|\bs{y}-\bs{x}\|_2^2\mid \forall n\in[N]\colon 
		\bs{w}_n^{\mathsf{T}}\bs{x}\geq \bs{w}_n^{\mathsf{T}}\bs{y},\ \bs{w}_k^{\mathsf{T}}\bs{y}\leq 0\}~. \notag
	\end{align}
	Similarly, we have $\Int(C)=\bigcup_{k\in[N]}\{\bs{y}\in\R^D\mid \bs{w}_k^{\mathsf{T}}\bs{y}\leq 0\}$ so that
	\begin{equation}\label{eq:b2problemg}
		d(\bs{x},(\Int(C))^c)^2=\min_{k\in[N]}g_k(\bs{x}), \text{ where }g_k(\bs{x})\coloneqq \inf\{\|\bs{y}-\bs{x}\|_2^2\mid \bs{w}_k^{\mathsf{T}}\bs{y}\leq 0\}~.
	\end{equation}
	Let us fix $k\in[N]$. Note that the optimization problems that define $f_k(\bs{x}), g_k(\bs{x})$ are quadratic and Karush-Kuhn-Tucker conditions characterize optimality. Let $\bs{y}\in \R^D$ be optimal for the calculation of $f_k(\bs{x})$. Hence, there exists $\bs{\lambda}\in\R^{N+1}_+$ such that the following conditions are satisfied:
	\begin{align}
		&2(\bs{y}-\bs{x})+\sum_{n\in [N]}\lambda_n\bs{w}_n +\lambda_{N+1}\bs{w}_k=0~,\label{FOCb2}\\
		&\forall n\in [N]\colon \lambda_n\bs{w}_n^{\mathsf{T}}(\bs{y}-\bs{x})=0~,\label{cs1b2}\\
		& \lambda_{N+1}\bs{w}_k^{\mathsf{T}}\bs{y}=0~.\label{cs2b2}
	\end{align}
	We claim that $\lambda_n=0$ for each $n\in[N]$. To see this, suppose that $\lambda_{\bar{n}}>0$ for some $\bar{n}\in[N]$. Then, \eqref{cs1b2} implies that $\bs{w}_{\bar{n}}^{\mathsf{T}}(\bs{y}-\bs{x})=0$. Combining this with \eqref{FOCb2}, we get
	\[
	\bs{w}_{\bar{n}}^{\mathsf{T}}(\bs{x}-\bs{y})=\frac12\Big(\sum_{n\in [N]\setminus\{k\}}\lambda_n \bs{w}_{\bar{n}}^{\mathsf{T}}\bs{w}_n+(\lambda_{N+1}+\lambda_k)\bs{w}_{\bar{n}}^{\mathsf{T}}\bs{w}_k \Big)=0~.
	\]
	By the assumption on $\bs{w}_1,\ldots,\bs{w}_N$, we have $\bs{w}_{\bar{n}}^{\mathsf{T}}\bs{w}_n\geq 0$ for every $n\in [N]$. Hence, all terms in the above sum are nonnegative. It follows that these terms are indeed zero, that is,
	\[
	\forall n\in [N]\setminus\{k\}\colon \lambda_n\bs{w}_{\bar{n}}^{\mathsf{T}}\bs{w}_n=0~,\quad (\lambda_{N+1}+\lambda_k)\bs{w}_{\bar{n}}^{\mathsf{T}}\bs{w}_k=0~.
	\]
	If $\bar{n}\neq k$, then we have $\lambda_{\bar{n}}\bs{w}_{\bar{n}}^{\mathsf{T}}\bs{w}_{\bar{n}}=\lambda_{\bar{n}}\|\bs{w}_{\bar{n}}\|_2^2=\lambda_{\bar{n}}=0$, which is a contradiction. If $\bar{n}=k$, then we have $(\lambda_{N+1}+\lambda_{\bar{n}})\bs{w}_{\bar{n}}^{\mathsf{T}}\bs{w}_{\bar{n}}=(\lambda_{N+1}+\lambda_{\bar{n}})\|\bs{w}_{\bar{n}}\|_2^2=\lambda_{N+1}+\lambda_{\bar{n}}=0$ so that $\lambda_{N+1}=\lambda_{\bar{n}}=0$, which is also a contradiction. Hence, the claim holds.
	
	As a consequence of the claim, \eqref{FOCb2} and \eqref{cs2b2} yield the equations
	\begin{align}
		& 2(\bs{y}-\bs{x})+\lambda_{N+1}\bs{w}_k=0~, \label{FOCdenb2}\\
		& \lambda_{N+1}\bs{w}_k^{\mathsf{T}}\bs{y}=0~,\label{csdenb2}
	\end{align}
	which establish Karush-Kuhn-Tucker conditions for the calculation of $g_k(\bs{x})$. In particular, $f_k(\bs{x})=g_k(\bs{x})=\|\bs{y}-\bs{x}\|_2^2$. It follows that $d(\bs{x},(\Int(C))^c\cap (\bs{x}-C))=d(\bs{x},(\Int(C))^c)$.
\end{proof}

\begin{corollary}\label{cor:conjCsup}
	(Theorem \ref{thm:beta}(ii)) Suppose that $C\supseteq \R^D_+$. Then, $\beta_1=\beta_2=1$.
\end{corollary}

\begin{proof}
	Since $C\supseteq \R^D_+$, it follows from the definition in \eqref{eq:dualcone} that $C^+\subseteq \R^D_+$. Since $\bs{w}_1,\ldots,\bs{w}_N\in C^+$, we have $\bs{w}_{n}^{\mathsf{T}}\bs{w}_k\geq 0$ for each $n,k\in [N]$. Propositions \ref{prop:conjCsup1} and \ref{prop:conjCsup2} yield the result.
\end{proof}

\subsubsection{Computation of $\beta_1$ for general ordering cones}

\revc{In the general case, without any additional assumptions on the ordering cone $C$ (besides polyhedrality), we will prove that $\beta_1$ can be calculated by solving a global optimization problem over a compact set. We begin with a duality lemma.
	
	\begin{lemma}\label{lem:beta1dual}
		For each $\bs{x}\in \R^D$, we have
		\begin{equation}
			d(\bs{x},(\bs{x}+C)\cap C)=\sup\Bigg\{\sum_{n\in[N]}\lambda_n(\bs{w}_n^{\mathsf{T}}\bs{x})^-\mid \Big\|\sum_{n\in[N]}\lambda_n\bs{w}_n\Big\|_2\leq 1~,\ \bs{\lambda}\in\R^N_+\Bigg\},\label{eq:beta1dual}
		\end{equation}
		In particular, $\bs{x}\mapsto d(\bs{x},(\bs{x}+C)\cap C)$ 
		is a continuous convex function on $\R^D$.
	\end{lemma}
	
	\begin{proof}
		Let $\bs{x}\in \R^D$. Recall that 
		\[
		d(\bs{x},(\bs{x}+C)\cap C)=\inf\{\|\bs{y}-\bs{x}\|_2\mid \forall n\in[N]\colon \bs{w}_n^{\mathsf{T}}\bs{y}\geq (\bs{w}_n^{\mathsf{T}}\bs{x})^+\}~.
		\]
		By strong duality for convex optimization, we have
		\[
		d(\bs{x},(\bs{x}+C)\cap C)=\sup_{\bs{\lambda}\in\R^N_+}h_1(\bs{\lambda})~,
		\]
		where
		\[
		h_1(\bs{\lambda}):=\inf_{\bs{y}\in\R^D}\Big(\|\bs{y}-\bs{x}\|_2-\sum_{n\in[N]}\lambda_n\bs{w}_n^{\mathsf{T}}\bs{y}+\sum_{n\in[N]}\lambda_n(\bs{w}_n^{\mathsf{T}}\bs{x})^+\Big)~.
		\]
		With a change-of-variables through $\bs{u}=\bs{y}-\bs{x}$, we get
		\begin{align*}
			h_1(\bs{\lambda})&=\inf_{\bs{u}\in\R^D}\Big(\|\bs{u}\|_2-\sum_{n\in[N]}\lambda_n\bs{w}_n^{\mathsf{T}}(\bs{u}+\bs{x})+\sum_{n\in[N]}\lambda_n(\bs{w}_n^{\mathsf{T}}\bs{x})^+\Big)\\
			&=\inf_{\bs{u}\in\R^D}\Big(\|\bs{u}\|_2-\sum_{n\in[N]}\lambda_n\bs{w}_n^{\mathsf{T}}\bs{u}\Big)+\sum_{n\in[N]}\lambda_n(\bs{w}_n^{\mathsf{T}}\bs{x})^-\\
			&=\begin{cases}\sum_{n\in[N]}\lambda_n(\bs{w}_n^{\mathsf{T}}\bs{x})^-&\text{if }\big\|\sum_{n\in[N]}\lambda_n\bs{w}_n\big\|_2\leq 1~,\\ -\infty &\text{else}~,\end{cases}
		\end{align*}
		where the last equality involves the routine calculation of the conjugate of the $\ell_2$-norm. Hence, \eqref{eq:beta1dual} follows. Clearly, $\bs{x}\mapsto \sum_{n\in[N]}\lambda_n(\bs{w}_n^{\mathsf{T}}\bs{x})^-$ is a convex function for every $\bs{\lambda}\in \R^N_+$. Then, by \eqref{eq:beta1dual}, $\bs{x}\mapsto d(\bs{x},(\bs{x}\cap C)\cap C)$ is convex as a supremum of convex functions. This function is also finite by definition. Hence, by \citet[Corollary 10.1.1]{rockafellar}, it is continuous as a convex finite function on $\R^D$.
	\end{proof}

	\begin{proposition}\label{prop:beta1comp}
		It holds
		\begin{equation}\label{eq:beta1comp}
			\beta_1=\sup\{d(-\bs{w},(-\bs{w}+C)\cap C)\mid \|\bs{w}\|_2=1,\ \bs{w}\in C^+\}~.
		\end{equation}
		In particular, $\beta_1<+\infty$.
	\end{proposition}
	
	\begin{proof}
		Let $\bs{p}_C\colon\R^D\to C$ be the projection mapping onto $C$, that is, for each $\bs{x}\in\R^D$, $\bs{p}_C(\bs{x})\in\R^D$ is the unique point in $C$ for which $d(\bs{x},C)=\|\bs{p}_C(\bs{x})-\bs{x}\|_2$. Clearly, for $\bs{x}\in C^c$, we have $\bs{p}_C(\bs{x})\in\bd(C)$.
		
		Let $\bs{x}\in C^c$. Then, \eqref{FOCden}, \eqref{csden} hold with $\bs{y}=\bs{p}_C(\bs{x})$ for some $\bs{\lambda}=\bs{\lambda}^L\in\R^N_+$. Let us define $\bar{\bs{x}}:=\frac{1}{d(\bs{x},C)}\bs{x}$. Note that $\bs{w}_n^{\mathsf{T}}\bar{\bs{x}}=\frac{1}{d(\bs{x},C)}\bs{w}_n^{\mathsf{T}}\bs{x}$ for every $n\in[N]$. Hence, $I(\bs{x})=I(\bar{\bs{x}})$. Let us define $\bar{\bs{y}}_L:=\frac{1}{d(\bs{x},C)}\bs{p}_C(\bs{x})$ and $\bar{\bs{\lambda}}^L:=\frac{1}{d(\bs{x},C)}\bs{\lambda}^L$. By \eqref{FOCden} for $(\bs{p}_C(\bs{x}),\bs{\lambda}^L)$, we have
		\begin{equation}\label{yLKKT1}
			\bar{\bs{y}}_L-\bar{\bs{x}}=\frac{1}{d(\bs{x},C)}(\bs{p}_C(\bs{x})-\bs{x})=\frac{1}{2}\sum_{n\in[N]}\frac{\lambda_n^L}{d(\bs{x},C)}\bs{w}_n=\frac12\sum_{n\in[N]}\bar{\lambda}_n^L\bs{w}_n~.
		\end{equation}
		Moreover, for each $n\in[N]$, by \eqref{csden} for $(\bs{p}_C(\bs{x}),\bs{\lambda}^L)$, we have
		\begin{equation}\label{yLKKT2}
			\bar{\lambda}_n^L\bs{w}_n^{\mathsf{T}}\bar{\bs{y}}=\frac{\lambda_n^L}{d(\bs{x},C)^2}\bs{w}_n^{\mathsf{T}}\bs{p}_C(\bs{x})=0~.
		\end{equation}
		Hence, \eqref{yLKKT1} and \eqref{yLKKT2} establish Karush-Kuhn-Tucker conditions for $(\bar{\bs{y}}_L,\bar{\bs{\lambda}}^L)$ so that $\bar{\bs{y}}_L=\bs{p}_C(\bar{\bs{x}})$ and 
		\[
		d(\bar{\bs{x}},C)=\|\bar{\bs{y}}_L-\bar{\bs{x}}\|_2=\frac{1}{d(\bs{x},C)}\|\bs{p}_C(\bs{x})-\bs{x}\|_2=1~.
		\]
		
		Next, note that \eqref{FOC}, \eqref{cs1}, \eqref{cs2} hold for some $\bs{y}=\bs{y}_U\in\R^D$, $\bs{\lambda}=\bs{\lambda}^U\in\R^N_+$. Let us define $\bar{\bs{y}}_U:=\frac{1}{d(\bs{x},C)}\bs{y}_U$ and $\bar{\bs{\lambda}}^U:=\frac{1}{d(\bs{x},C)}\bs{\lambda}^U$. By \eqref{FOC} for $(\bs{y}_U,\bs{\lambda}^U)$, we have
		\begin{equation}\label{yUKKT1}
			\bar{\bs{y}}_U-\bar{\bs{x}}=\frac{1}{d(\bs{x},C)}(\bs{y}_U-\bs{x})=\frac12\sum_{n\in[N]}\frac{\lambda_n^U}{d(\bs{x},C)}\bs{w}_n=\frac12\sum_{n\in[N]}\bar{\lambda}^U_n\bs{w}_n~.
		\end{equation}
		Let $n\in I(\bar{\bs{x}})=I(\bs{x})$. Then, by \eqref{cs1} for $(\bs{y}_U,\bs{\lambda}^U)$, we have
		\begin{equation}\label{yUKKT2}
			\bar{\lambda}^U_n\bs{w}_n^{\mathsf{T}}\bar{\bs{y}}^U=\frac{\lambda_n^U}{d(\bs{x},C)^2}\bs{w}_n^{\mathsf{T}}\bs{y}_U=0~.
		\end{equation}
		Let $n\in I(\bar{\bs{x}})^c=I(\bs{x})^c$. Then, by \eqref{cs2} for $(\bs{y}_U,\bs{\lambda}^U)$, we have
		\begin{equation}\label{yUKKT3}
			\bar{\lambda}^U_n\bs{w}_n^{\mathsf{T}}(\bar{\bs{y}}^U-\bar{\bs{x}})=\frac{\lambda_n^U}{d(\bs{x},C)^2}\bs{w}_n^{\mathsf{T}}(\bs{y}_U-\bs{x})=0~.
		\end{equation}
		Hence, \eqref{yUKKT1}, \eqref{yUKKT2}, \eqref{yUKKT3} establish Karush-Kuhn-Tucker conditions for $(\bar{\bs{y}}_U,\bar{\bs{\lambda}}^U)$ so that
		\[
		d(\bar{\bs{x}},(\bar{\bs{x}}+C)\cap C)=\|\bar{\bs{y}}_U-\bar{\bs{x}}\|_2=\frac{\|\bs{y}_U-\bs{x}\|_2}{d(\bs{x},C)}=\frac{d(\bs{x},(\bs{x}+C)\cap C)}{d(\bs{x},C)}~.
		\]
		Since we also have $d(\bar{\bs{x}},C)=1$, we obtain
		\[
		\frac{d(\bs{x},(\bs{x}+C)\cap C)}{d(\bs{x},C)}=\frac{d(\bar{\bs{x}},(\bar{\bs{x}}+C)\cap C)}{d(\bar{\bs{x}},C)}~.
		\]
		Therefore, we may restrict the calculation of $\beta_1$ to the set of points that are of unit distance to the ordering cone, i.e.,
		\[
		\beta_1=\sup\{d(\bs{x},(\bs{x}+C\cap C))\mid x\in C^c,\ d(\bs{x},C)=1\}~.
		\]

		Note that a vector $\bs{y}\in\R^D$ is in $\bd(C)$ if and only if $\bs{y}\in C$ and $\bs{w}_n^{\mathsf{T}}\bs{y}>0$ for at least one $n\in [N]$. In view of this observation, we may write $\bd(C)$ as the union of the faces of $C$, i.e.,
		\[
		\bd(C)=\bigcup_{I\in 2^{[N]}\setminus\{\emptyset\}}F(I)~, \text{ where }F(I):=\{\bs{y}\in C\mid \forall n\in I \colon \bs{w}_n^{\mathsf{T}}\bs{y}=0,\ \forall n\in I^c\colon \bs{w}_n^{\mathsf{T}}\bs{y}>0\}~.
		\]
		For each $I\in 2^{[N]}\setminus\{\emptyset\}$, let
		\[
		\mathcal{X}(I)\coloneqq \{\bs{x}\in C^c\mid \bs{p}_C(\bs{x})\in F(I)\},\quad C^+(I)\coloneqq \Big\{\sum_{n\in I}\lambda_n\bs{w}_n\mid \bs{\lambda}\in\R^N_+\Big\}~.
		\]
		We claim that
		\begin{equation}\label{eq:claimface}
			\mathcal{X}(I)=F(I)-C^+(I)~.
		\end{equation}
		To see this, first let us fix $\bs{x}\in\mathcal{X}(I)$, i.e., $\bs{p}_C(\bs{x})\in F(I)$. Note that $\bs{y}=\bs{p}_C(\bs{x})$ satisfies \eqref{FOCden}, \eqref{csden} for some $\bs{\lambda}\in\R^N_+$. Since $\bs{p}_C(\bs{x})\in F(I)$, it follows from \eqref{csden} that $\lambda_n=0$ for each $n\in I^c$. Then, $\bs{p}_C(\bs{x})-\bs{x}=\frac12 \sum_{n\in I}\lambda_n\bs{w}_n\in C^+(I)$. Hence, $\bs{x}-\bs{p}_C(\bs{x})\in -C^+(I)$. Conversely, let $\bs{y}\in F(I)$, $\bs{\lambda}\in\R^N_+$, and define $\bs{x}\coloneqq\bs{y}-\sum_{n\in I}\lambda_n\bs{w}_n$. Let us define $\bar{\bs{\lambda}}\in\R^N_+$ by $\bar{\lambda}_n:=\lambda_n$ for each $n\in I$, and by $\bar{\lambda}_n:=0$ for each $n\in I^c$. Then, it is clear that the pair $(\bs{y},\bar{\bs{\lambda}})$ satisfies \eqref{FOCden}, \eqref{csden} for the point $\bs{x}$. Hence, we conclude that $\bs{p}_C(\bs{x})=\bs{y}$. Therefore, $\bs{x}\in \mathcal{X}(I)$, which completes the proof of \eqref{eq:claimface}. The last part of the proof also shows that $d(\bs{x},C)=\|\bs{w}\|_2$ whenever $\bs{x}=\bs{y}-\bs{w}$ with $\bs{y}\in F(I)$ and $\bs{w}\in C^+(I)$.
		
		Note that we may write $C^c=\bigcup_{I\in 2^{[N]}\setminus\{\emptyset\}}\mathcal{X}(I)$ as a disjoint finite union. Accordingly, we have
		\[
		\beta_1= \max_{I\in 2^{[N]}\setminus\{\emptyset\}}\beta_1(I)~,\text{ where }\beta_1(I):=\sup\{d(\bs{x},(\bs{x}+C)\cap C)\mid d(\bs{x},C)=1,\,\bs{x}\in \mathcal{X}(I)\}~.
		\]
		Let us fix $I\in 2^{[N]}\setminus\{\emptyset\}$ and $\bs{x}\in\mathcal{X}(I)$ with $d(\bs{x},C)=1$. We may write $\bs{x}=\bs{p}_C(\bs{x})-\bs{w}$ for some $\bs{w}\in C^+(I)$ and $d(\bs{x},C)=\|\bs{w}\|_2=1$. Let us define $\bar{\bs{x}}:=-\bs{w}$. Let $n\in [N]$. Since $\bs{p}_C(\bs{x})\in F(I)$, we have
		\[
		\bs{w}_n^{\mathsf{T}}\bar{\bs{x}}=\bs{w}_n^{\mathsf{T}}(\bs{x}-\bs{p}_C(\bs{x}))=\bs{w}_n^{\mathsf{T}}\bs{x}
		\]
		if $n\in I$, and
		\[
		\bs{w}_n^{\mathsf{T}}\bar{\bs{x}}=\bs{w}_n^{\mathsf{T}}(\bs{x}-\bs{p}_C(\bs{x}))<\bs{w}_n^{\mathsf{T}}\bs{x}
		\]
		if $n\in I^c$. Hence, for every $n\in [N]$, we have $\bs{w}_n^{\mathsf{T}}\bar{\bs{x}}\leq \bs{w}_n^{\mathsf{T}}\bs{x}$ so that $(\bs{w}_n^{\mathsf{T}}\bar{\bs{x}})^-\geq (\bs{w}_n^{\mathsf{T}}\bs{x})^-$. Hence, by Lemma \ref{lem:beta1dual},
		\begin{align*}
			d(\bs{x},(\bs{x}+C)\cap C)&=\sup\Bigg\{\sum_{n\in[N]}\lambda_n(\bs{w}_n^{\mathsf{T}}\bs{x})^-\mid \Big\|\sum_{n\in[N]}\lambda_n\bs{w}_n\Big\|_2\leq 1,\ \bs{\lambda}\in\R^N_+\Bigg\}\\
			&\leq \sup\Bigg\{\sum_{n\in[N]}\lambda_n(\bs{w}_n^{\mathsf{T}}\bar{\bs{x}})^-\mid \Big\|\sum_{n\in[N]}\lambda_n\bs{w}_n\Big\|_2\leq 1,\ \bs{\lambda}\in\R^N_+\Bigg\}\\
			&=d(\bar{\bs{x}},(\bar{\bs{x}}+C)\cap C)~.
		\end{align*}
		Moreover, note that $p_C(\bar{\bs{x}})=0\in F([N])$, $\bar{\bs{x}}\in \mathcal{X}([N])$, and $d(\bar{\bs{x}},C)=\|\bs{w}\|_2=1$. As $\bar{\bs{x}}\in\mathcal{X}([N])$ provides a higher objective value than $\bs{x}\in\mathcal{X}([I])$, we can ignore all choices of $I$ but the case $I=[N]$ in the computation of $\beta_1$:
		\begin{equation}\label{eq:beta1new}
			\beta_1=\sup\{d(\bs{x},(\bs{x}+C)\cap C)\mid d(\bs{x},C)=1,\ \bs{x}\in\mathcal{X}([N])\}~.
		\end{equation}
		Note that $F([N])=\{0\}$ since $C$ is assumed to be pointed, and $C^+([N])=C^+$. Hence, $\mathcal{X}([N])=F([N])-C^+(I)=-C^+$ by \eqref{eq:claimface}. Therefore, \eqref{eq:beta1new} becomes
		\begin{align*}
			\beta_1&=\sup\{d(-\bs{w},(-\bs{w}+C)\cap C)\mid d(-\bs{w},C)=1,\ \bs{w}\in C^+\}\\
			&=\sup\{d(-\bs{w},(-\bs{w}+C)\cap C)\mid \|\bs{w}\|_2=1,\ \bs{w}\in C^+\}~.
		\end{align*}
		Hence, \eqref{eq:beta1comp} follows.
		
		Finally, note that $\{\bs{w}\in C^+\mid \|\bs{w}\|_2=1\}$ is a compact set. Moreover, $\bs{w}\mapsto d(-\bs{w},(-\bs{w}+C)\cap C)$ is a continuous function by Lemma \ref{lem:beta1dual}. Hence, there exists $\bs{w}^\ast\in C^+$ such that $\|\bs{w}^\ast\|_2=1$ and $\beta_1=d(-\bs{w}^\ast,(-\bs{w}^\ast+C)\cap C)<+\infty$.
	\end{proof}
	
	The optimization problem in \eqref{eq:beta1comp} consists of maximizing a finite convex objective function over a compact set; hence, it is a global (non-convex) optimization problem. Nevertheless, the objective function is indeed Lipschitz over $\{\bs{w}\in C^+\mid \|\bs{w}\|_2=1\}$ by \citet[Theorem 10.4]{rockafellar}. Hence, some classical techniques of global optimization (e.g., \cite{goldstein}) can be used to calculate the actual value of $\beta_1$ (approximately). To be able to use these methods, one still needs to find a Lipschitz constant for the function $\bs{w}\mapsto d(-\bs{w},(-\bs{w}+C)\cap C)$ over $\{\bs{w}\in C^+\mid \|\bs{w}\|_2=1\}$. A more convenient alternative would be to use the same methods of global optimization for the following dual representation of $\beta_1$.
	
	\begin{theorem}\label{thm:beta1}
		It holds
		\begin{align*}
			\beta_1&= \sup_{\substack{\bs{w}\in C^+\colon\\ \|\bs{w}\|_2=1}}\sup_{\substack{\bs{\lambda}\in\R^N_+\colon \\ \|\sum_{n\in[N]}\lambda_n\bs{w}_n\|_2\leq 1}}\sum_{n\in[N]}\lambda_n(\bs{w}_n^{\mathsf{T}}\bs{w})^+\\
			&= \max_{I\in 2^{[N]}}\sup_{\substack{\bs{w}\in C^+\cap C(I)\colon\\ \|\bs{w}\|_2=1}}\sup_{\substack{\bs{\lambda}\in\R^N_+\colon \\ \|\sum_{n\in[N]}\lambda_n\bs{w}_n\|_2\leq 1}}\sum_{n\in I}\lambda_n\bs{w}_n^{\mathsf{T}}\bs{w}~,
		\end{align*}
		where
		\[
		C(I):=\{\bs{y}\in\R^D\mid \forall n\in I\colon \bs{w}_n^{\mathsf{T}}\bs{y}\geq 0,\ \forall n\in I^c\colon \bs{w}_n^{\mathsf{T}}\bs{y}\leq 0\}~.
		\]
	\end{theorem}
	
	\begin{proof}
		By Lemma \ref{lem:beta1dual}, for each $\bs{w}\in C^+$, we have
		\[
		d(-\bs{w},(-\bs{w}+C)\cap C)=\sup\Bigg\{\sum_{n\in[N]}\lambda_n(\bs{w}_n^{\mathsf{T}}\bs{w})^+\mid \Big\|\sum_{n\in[N]}\lambda_n\bs{w}_n\Big\|_2\leq 1,\ \bs{\lambda}\in\R^N_+\Bigg\}~.
		\]
		Combining this and Proposition \ref{prop:beta1comp}, the first equality in the theorem follows. The second equality follows immediately by decomposing $C^+$ as $C^+=\bigcup_{I\in 2^{[N]}}(C^+\cap C(I))$ and noting that $\sum_{n\in[N]}\lambda_n(\bs{w}_n^{\mathsf{T}}\bs{w})^+=\sum_{n\in I}\lambda_n\bs{w}_n^{\mathsf{T}}\bs{w}$ whenever $\bs{w}\in C(I)$.
	\end{proof}
	
	We complete this subsection by showing that, for each $I\in 2^{[N]}$, the objective function of the second problem in Theorem \ref{thm:beta1} is jointly Lipschitz with a computable Lipschitz constant. This will be achieved by the next proposition. In what follows, $\|\cdot\|_1$ denotes the $\ell_1$-norm.
	
	\begin{proposition}\label{lem:beta1lambda}
		(i) Let $v^\ast:= \inf\{\|\sum_{n\in[N]}\lambda \bs{w}_n\|_2\mid \|\bs{\lambda}\|_1=1,\ \bs{\lambda}\in\R^N_+\}$. Then, $v^\ast>0$.\\
		(ii) The set $\{\bs{\lambda}\in\R^N_+\mid \|\sum_{n\in [N]}\lambda_n\bs{w}_n\|_2\leq 1\}$ is a nonempty compact subset of $\R^N$.\\
		(iii) The function $(\bs{w},\bs{\lambda})\mapsto \sum_{n\in I}\lambda_n\bs{w}_n^{\mathsf{T}}\bs{w}$ is $(\frac{1}{v^\ast}\vee 1)$-Lipschitz with respect to the $\ell_2\times\ell_1$-norm on $\R^D\times\R^N$.
	\end{proposition}
	
	\begin{proof}
		(i) Let us define $\Lambda:=\{\bs{\lambda}\in\R^N_+\mid \|\bs{\lambda}\|_1=1\}$, which is a compact set. The function $\bs{\lambda}\mapsto \|\sum_{n\in [N]}\lambda_n\bs{w}_n\|_2$ is continuous, hence it attains its minimum over $\Lambda$ at some $\bs{\lambda}^\ast\in\Lambda$. We claim that the corresponding minimum value $v^\ast=\|\sum_{n\in [N]}\lambda_n^\ast\bs{w}_n\|_2$ is nonzero. Indeed, having $v^\ast=0$ would imply $\sum_{n\in [N]}\lambda_n^\ast\bs{w}_n=0$, which is not possible since $\bs{w}_1,\ldots,\bs{w}_N$ are the generating vectors of the dual cone $C^+$ and $C^+$ is pointed (as $C$ is solid). Therefore, the claim holds.
		
		(ii) By (i), we have
		\begin{equation}\label{eq:lambda}
			\Big\|\sum_{n\in [N]}\lambda_n\bs{w}_n\Big\|_2\geq v^\ast>0
		\end{equation}
		for every $\bs{\lambda}\in \Lambda$. Finally, for an arbitrary vector $\bs{\lambda}\in\R^N_+\setminus\{\bs{0}\}$, we have
		\begin{equation}\label{eq:lambda2}
			\Big\|\sum_{n\in [N]}\lambda_n\bs{w}_n\Big\|_2\geq v^\ast \|\bs{\lambda}\|_1~,
		\end{equation}
		which follows by applying \eqref{eq:lambda} to $\bs{\lambda}/\|\bs{\lambda}\|_1\in\Lambda$. Moreover, \eqref{eq:lambda2} holds for $\bs{\lambda}=\bs{0}$ trivially. Therefore, \eqref{eq:lambda2} holds for every $\bs{\lambda}\in \R^N_+$. In particular,
		\[
		\sup\Bigg\{\|\bs{\lambda}\|_1\mid \Big\|\sum_{n\in [N]}\lambda_n\bs{w}_n\Big\|_2\leq 1,\ \bs{\lambda}\in\R^N_+\Bigg\}\leq \frac{1}{v^\ast}<+\infty~,
		\]
		which shows the boundedness of $\{\bs{\lambda}\in\R^N_+\mid \|\sum_{n\in [N]}\lambda_n\bs{w}_n\|_2\leq 1\}$. Trivially, this set is also closed. Hence, compactness follows.
		
		(iii) For every $\bs{\lambda},\bar{\bs{\lambda}}\in\R^N_+$ such that $\|\sum_{n\in[N]}\lambda_n\bs{w}_n\|_2\leq 1$, $\|\sum_{n\in[N]}\bar{\lambda}_n\bs{w}_n\|_2\leq 1$ and $\bs{w},\bar{\bs{w}}\in C^+$ such that $\|\bs{w}\|_2=\|\bar{\bs{w}}\|_2=1$, we observe that
		\begin{align*}
			\Big|\sum_{n\in I}\lambda_n\bs{w}_n^{\mathsf{T}}\bs{w}-\sum_{n\in I}\bar{\lambda}_n\bs{w}_n^{\mathsf{T}}\bar{\bs{w}}\Big|&\leq \Big\|\sum_{n\in I}\lambda_n\bs{w}_n\Big\|_2\|\bs{w}-\bar{\bs{w}}\|_2+\Big\|\sum_{n\in I}(\lambda_n-\bar{\lambda}_n)\bs{w}_n\Big\|_2\|\bar{\bs{w}}\|_2\\
			&\leq \sum_{n\in I}\lambda_n \|\bs{w}-\bar{\bs{w}}\|_2+ \sum_{n\in I}|\lambda_n-\bar{\lambda}_n|\\
			&\leq \frac{1}{v^\ast}\|\bs{w}-\bar{\bs{w}}\|_2+\|\bs{\lambda}-\bar{\bs{\lambda}}_n\|_1\\
			&\leq \Big(\frac{1}{v^\ast}\vee 1\Big)(\|\bs{w}-\bar{\bs{w}}\|_2+\|\bs{\lambda}-\bar{\bs{\lambda}}_n\|_1)~.
		\end{align*}
		Hence, the Lipschitz property follows. 
	\end{proof}
	
	The constant
	\[
		v^\ast = \inf\Bigg\{\Big\|\sum_{n\in[N]}\lambda \bs{w}_n\Big\|_2\mid \|\bs{\lambda}\|_1=1,\ \bs{\lambda}\in\R^N_+\Bigg\}
		=\inf\Bigg\{\Big\|\sum_{n\in[N]}\lambda \bs{w}_n\Big\|_2\mid \|\bs{\lambda}\|_1\leq 1,\ \bs{\lambda}\in\R^N_+\Bigg\}~,
	\]
	which appears in Proposition \ref{lem:beta1lambda}, can be calculated by solving a simple minimization problem whose objective function is quadratic and constraints can be linearized. Therefore, in view of Theorem \ref{thm:beta1}, to evaluate $\beta_1$ approximately, one can solve finitely many optimization problems with a compact feasible region and Lipschitz objective function with a known Lipschitz constant.
}

\subsubsection{Computation of $\beta_2$ for general ordering cones}

\revc{We will find an upper bound for $\beta_2$. For this, we first provide a simple formula for the calculation of the distance functions that appear in the definition of $\beta_2$.
	
	\begin{lemma}\label{lem:beta2}
		Let $\bs{x}\in (\Int(C))^c$. Then,
		\[
		d(\bs{x},(\Int(C))^c\cap(\bs{x}-C))=\min_{n\in[N]}\frac{\bs{w}_n^{\mathsf{T}}\bs{x}}{\alpha_n}~,\qquad
		d(\bs{x},(\Int(C))^c)=\min_{n\in[N]}\bs{w}_n^{\mathsf{T}}\bs{x}~.
		\]
	\end{lemma}
	
	\begin{proof}
		To see the first identity, note that replacing $\bs{\mu}_j-\bs{\mu}_i$ with $\bs{x}$ in Proposition \ref{prop:m}(ii,iv) gives $d(\bs{x},(\Int(C))^c\cap(\bs{x}-C))= \min_{n\in[N]}\frac{(\bs{w}_n^{\mathsf{T}}\bs{x})^+}{\alpha_n}$. Since $\bs{x}\in \Int(C)$, we have $\bs{w}_n^{\mathsf{T}}\bs{x}>0$ for each $n\in[N]$. Hence, the first identity follows. Alternatively, one can use \eqref{eq:b2problemf} and formulate the Lagrange dual problem of $f_n(\bs{x})$ for each $n\in[N]$, which yields the same identity.
		
		To prove the second identity, let us use \eqref{eq:b2problemg}. By strong duality for convex optimization, for each $n\in [N]$, we have
		\begin{align*}
			g_n(\bs{x})&=\sup_{\lambda\geq 0}\inf_{\bs{y}\in\R^D}\big(\|\bs{y}-\bs{x}\|_2+\lambda \bs{w}_n^{\mathsf{T}}\bs{y}\big)\\
			&=\sup_{\lambda\geq 0}\inf_{\bs{u}\in\R^D}\big(\|\bs{u}\|_2+\lambda \bs{w}_n^{\mathsf{T}}(\bs{u}+\bs{x})\big)\\
			&=\sup_{\lambda\geq 0}\Big(\lambda\bs{w}_n^{\mathsf{T}}\bs{x}-\sup_{\bs{u}\in\R^D}\big(-\lambda\bs{w}_n^{\mathsf{T}}\bs{u}-\|\bs{u}\|_2\big)\Big)\\
			&=\sup_{\lambda\geq 0}\begin{cases}\lambda\bs{w}_n^{\mathsf{T}}\bs{x}&\text{if }\lambda=\|\lambda \bs{w}_n\|\leq 1\\ -\infty&\text{else}\end{cases}\\
			&=\sup_{\lambda\in [0,1]}\lambda \bs{w}_n^{\mathsf{T}}\bs{x}=\bs{w}_n^{\mathsf{T}}\bs{x}~,
		\end{align*}
		where the last equality holds since $\bs{w}_n^{\mathsf{T}}\bs{x}>0$. Hence, the second identity follows.
	\end{proof}
	
	\begin{theorem}\label{thm:beta2comp}
		It holds
		\[
		\beta_2\leq \frac{1}{\min_{n\in[N]}\alpha_n}<+\infty~.
		\]
	\end{theorem}
	
	\begin{proof}
		Let $n^\ast\in[N]$ be such that $\alpha_{n^\ast}=\min_{n\in[N]}\alpha_n$. By the definition of $\beta_2$ and Lemma \ref{lem:beta2}, we have
		\begin{align*}
			\beta_2 = \sup_{\bs{x}\in\Int(C)}\frac{d(\bs{x},(\Int(C))^c\cap(\bs{x}-C))}{d(\bs{x},(\Int(C))^c)}  & =\sup_{\bs{x}\in\Int(C)}\frac{\min_{n\in[N]}\frac{\bs{w}_n^{\mathsf{T}}\bs{x}}{\alpha_n}}{\min_{n\in[N]}\bs{w}_n^{\mathsf{T}}\bs{x}} \\ 
			& \leq \frac{1}{\alpha_{n^\ast}}\sup_{\bs{x}\in\Int(C)}\frac{\min_{n\in[N]}\bs{w}_n^{\mathsf{T}}\bs{x}}{\min_{n\in[N]}\bs{w}_n^{\mathsf{T}}\bs{x}} = \frac{1}{\alpha_{n^\ast}}~.
		\end{align*}
	\end{proof}
	Note that Proposition \ref{prop:beta1comp} and Theorem \ref{thm:beta2comp} yield Theorem \ref{thm:beta}(i).
}

\subsection{Proof of Proposition \ref{prop:m}}
(i) To get a contradiction, suppose that $m(i,j)=+\infty$. Hence, the set whose infimum is calculated in \eqref{eq:m} is empty, that is, for every $s\geq 0$ and $\bs{u}\in B(\bs{0},1)\cap C$, we have $\bs{\mu}_i+s\bs{u}\in \bs{\mu}_j-\Int(C)$. Since $\{s\bs{u}\mid s\geq 0, \bs{u}\in B(\bs{0},1)\cap C\}=C$, we have $\bs{\mu}_i + C\subseteq \bs{\mu}_j - \Int(C)\subseteq \bs{\mu}_j-C$. Hence, $\bs{\mu}_i-\bs{\mu}_j+C\subseteq -C$ so that $C\cap(-C)\supseteq C\cap (\bs{\mu}_i-\bs{\mu}_j +C)$. The latter intersection consists of all points $\bs{x}\in\R^D$ such that $\bs{0}\preceq_C \bs{x}$ and $\bs{\mu}_i-\bs{\mu}_j \preceq_C \bs{x}$. Since $C$ is a solid cone, there are infinitely many such points, which contradicts with $C\cap (-C)=\{\bs{0}\}$. Hence, $m(i,j)<+\infty$.\\
(ii) Note that
\begin{align*}
	m(i,j) = & \inf\{ s \geq 0 \mid  \exists \bs{u} \in B(\bs{0},1)\cap C \colon \bs{\mu}_i + s \bs{u} \in (\bs{\mu}_j - \Int(C))^c \} \\
	& \inf\{ s \geq 0 \mid  \exists \bs{u} \in B(\bs{0},1)\cap C \colon \bs{\mu}_i + s \bs{u} \in \bs{\mu}_j - (\Int(C))^c \} \\
	= & \inf_{\bs{c} \in (\Int (C))^c} \inf \{ s \geq 0 \mid \exists \bs{u} \in B(\bs{0},1)\cap C \colon  \bs{\mu}_j -\bs{\mu}_i -c = s \bs{u} \} \\
	= & \inf_{\bs{c} \in (\Int (C))^c\cap (\bs{\mu}_j-\bs{\mu}_i-C)} \inf \{ s \geq 0 \mid  s= \|\bs{\mu}_j -\bs{\mu}_i -c\|_2  \} \\
	= & \inf_{\bs{c} \in (\Int (C))^c\cap (\bs{\mu}_j-\bs{\mu}_i-C)} \|\bs{\mu}_j -\bs{\mu}_i -c\|_2 \\
	= & d(\bs{\mu}_j - \bs{\mu}_i, (\Int C)^c\cap (\bs{\mu}_j-\bs{\mu}_i-C))~.
\end{align*}
(iii) Since $(\Int(C))^c\cap (\bs{\mu}_j-\bs{\mu}_i-C)$ is a closed set, by the well-known properties of distance function, we have $\bs{\mu}_j-\bs{\mu}_i\in (\Int(C))^c\cap (\bs{\mu}_j-\bs{\mu}_i-C)$ if and only if $m(i,j)=d(\bs{\mu}_j- \bs{\mu}_i, (\Int(C))^c\cap (\bs{\mu}_j-\bs{\mu}_i-C))=0$. Since $\bs{\mu}_j-\bs{\mu}_i \in \bs{\mu}_j-\bs{\mu}_i -C$ always holds, we have $\bs{\mu}_j-\bs{\mu}_i\in (\Int(C))^c$ if and only if $m(i,j)=0$. Since the former condition precisely means that $i\nprec_C j$, the desired equivalence follows.\\
(iv) First, suppose that $m(i,j)=0$, that is, $\bs{\mu}_j-\bs{\mu}_i\notin \Int(C)$. Since $\Int(C)=\{\bs{x}\in\R^D\mid \bs{W}\bs{x}>0\}$, there exists $\bar{n}\in [N]$ such that $\bs{w}_{\bar{n}}^{\mathsf{T}}(\bs{\mu}_j-\bs{\mu}_i)\leq 0$, that is, $(\bs{w}_{\bar{n}}^{\mathsf{T}}(\bs{\mu}_j-\bs{\mu}_i))^+=0$. Hence, $\min_{n\in [N]}(w_n^{\mathsf{T}}(\bs{\mu}_j-\bs{\mu}_i))^+/\alpha_n=0=m(i,j)$. Next, suppose that $m(i,j)>0$, that is, $\bs{\mu}_j-\bs{\mu}_i\in \Int(C)$. Note that we may write $m(i,j)=\inf R$, where
\begin{equation*}
	R:=\{s\geq 0\mid \exists \bs{u}\in B(\bs{0},1)\cap C\colon \bs{\mu}_i+s\bs{u}\notin \bs{\mu}_j-C\}~.
\end{equation*}
We show that $R$ is an interval that is unbounded from above. To that end, let $s\geq 0$ be such that $\bs{\mu}_i+s\bs{u}\notin \bs{\mu}_j-C$ for some $\bs{u}\in B(\bs{0},1)\cap C$. Let $s'>s$. We claim that $\bs{\mu}_i+s'\bs{u}\notin \bs{\mu}_j-C$. Suppose otherwise that $\bs{\mu}_i+s'\bs{u}\in \bs{\mu}_j-C$. Since $\bs{\mu}_i\in\bs{\mu}_j-\Int(C)\subseteq \bs{\mu}_j-C$ by supposition, and $\bs{\mu}_j-C$ is a convex set, we obtain
\begin{equation*}
	\bs{\mu}_i+s\bs{u}=\left(1-\frac{s}{s'}\right)\bs{\mu}_i + \frac{s}{s'}(\bs{\mu}_i+s'\bs{u})\in \bs{\mu}_j-C~,
\end{equation*}
which is a contradiction to the definition of $s$. Hence, the claim holds. Since $R$ is an interval in $[0,+\infty)$ that is unbounded from above, we have $\inf R = \sup ([0,+\infty)\setminus R)$. Therefore,
\begin{align*}
	m(i,j)=&  \sup\{s\geq 0\mid \forall \bs{u}\in B(\bs{0},1)\cap C\colon\bs{\mu}_i+s\bs{u}\in \bs{\mu}_j-C\}\\
	= & \sup\{ s \geq 0 \mid \forall \bs{u} \in B(\bs{0},1)\cap C, \forall n \in [N]\colon \bs{w}^{\mathsf{T}}_n (\bs{\mu}_j - \bs{\mu}_i - s \bs{u}) \geq 0 \} \\
	= & \sup\{ s \geq 0 \mid \forall \bs{u} \in B(\bs{0},1)\cap C, \forall n \in [N]\colon  \bs{w}^{\mathsf{T}}_n (\bs{\mu}_j - \bs{\mu}_i ) \geq s  \bs{w}^{\mathsf{T}}_n\bs{u} \} \\ 
	= & \sup\{ s \geq 0 \mid \forall n \in [N]\colon  \bs{w}^{\mathsf{T}}_n (\bs{\mu}_j - \bs{\mu}_i ) \geq s \sup_{\bs{u} \in B(\bs{0},1)\cap C} \bs{w}^{\mathsf{T}}_n \bs{u} \} \\
	= & \sup\{ s \geq 0 \mid \forall n\in [N]\colon  \bs{w}^{\mathsf{T}}_n (\bs{\mu}_j - \bs{\mu}_i ) \geq s \alpha_n\}\\
	= & \sup\left\{ s \geq 0 \mid \min_{n\in [N]} \frac{1}{\alpha_n}\bs{w}^{\mathsf{T}}_n (\bs{\mu}_j - \bs{\mu}_i ) \geq s \right\} \\
	= & \max \left\{ 0, \min_{n \in [N]} \frac{1}{\alpha_n}\bs{w}^{\mathsf{T}}_n (\bs{\mu}_j - \bs{\mu}_i ) \right\}\\
	=&\min_{n\in[N]}\frac{(\bs{w}^{\mathsf{T}}_n(\bs{\mu}_j-\bs{\mu}_i))^+}{\alpha_n}~,
\end{align*}
which completes the proof.

\subsection{Proof of Proposition \ref{prop:M}}

(i) We prove that the set whose infimum is calculated in \eqref{eq:M} is nonempty. To get a contradiction, suppose that for every $s\geq 0$ and $\bs{u}\in B(\bs{0},1)\cap C$, we have $\bs{\mu}_j-\bs{\mu}_i+s\bs{u}\notin C$, that is,
\begin{equation*}
	\Big(\bs{\mu}_j-\bs{\mu}_i+\bigcup_{s\geq 0}(B(\bs{0},s)\cap C)\Big)\cap C=\emptyset~.
\end{equation*}
However, we have $\bigcup_{s\geq 0}(B(\bs{0},s)\cap C)=C$; hence, we get a contradiction to the solidity of $C$ as in the proof of Proposition \ref{prop:m}(i). It follows that $M(i,j)<+\infty$.\\
(ii) By elementary calculations, we have
\begin{align*}
	M(i,j)
	=& \inf\{s \geq 0 \mid \exists \bs{u} \in B(\bs{0},1)\cap C, \exists \bs{c} \in C \colon \bs{\mu}_j + s \bs{u} = \bs{\mu}_i + \bs{c} \} \\
	=& \inf_{\bs{c} \in C} \inf\{s \geq 0 \mid \exists \bs{u} \in B(\bs{0},1)\cap C \colon  \bs{\mu}_j - \bs{\mu}_i - \bs{c}= -s \bs{u} \} \\
	=& \inf_{\bs{c} \in C\cap (\bs{\mu}_j-\bs{\mu}_i+C)} \inf\{s \geq 0 \mid  s = \| \bs{\mu}_j - \bs{\mu}_i - \bs{c} \|_2 \} \\
	=& \inf_{\bs{c} \in C\cap (\bs{\mu}_j-\bs{\mu}_i+C)} || \bs{\mu}_j - \bs{\mu}_i - \bs{c} ||_2\\
	=& d( \bs{\mu}_j - \bs{\mu}_i, C\cap (\bs{\mu}_j-\bs{\mu}_i+C))  ~.
\end{align*}
(iii) Since $C\cap (\bs{\mu}_j-\bs{\mu}_i+C)$ is a closed set, we have $M(i,j)=d(\bs{\mu}_j-\bs{\mu}_i,C\cap (\bs{\mu}_j-\bs{\mu}_i+C))=0$ if and only if $\bs{\mu}_j-\bs{\mu}_i\in C\cap (\bs{\mu}_j-\bs{\mu}_i+C)$. Since $\bs{\mu}_j-\bs{\mu}_i\in \bs{\mu}_j-\bs{\mu}_i+C$ is always the case, these conditions are also equivalent to $\bs{\mu}_j-\bs{\mu}_i\in C$, that is, $i\preceq_C j$.

\subsection{Proof of Theorem \ref{lowerboundthm}}

For each $s>0$, let us define
\[
A(s):=\bigcap_{\bs{u}\in B(\bs{0},1)\cap C}(s\bs{u}+C),\quad f(s):=d(\bs{0},A(s))=\inf\{\|\bs{z}\|\mid \bs{z}\in A(s)\}~.
\]
Note that $\bs{z}\in A(s)$ if and only if $\bs{w}_n^{\mathsf{T}}(\bs{z}-s\bs{u})\geq 0$ for every $n\in[N]$ and $\bs{u}\in B(\bs{0},1)$, which is equivalent to $\bs{w}_n^{\mathsf{T}}\bs{z}\geq s\alpha_n=s\sup_{\bs{u}\in B(\bs{0},1)\cap C}\bs{w}_n^{\mathsf{T}}\bs{u}$ for every $n\in [N]$. It follows that
\[
A(s)=\{\bs{z}\in\R^D\mid \forall n\in [N]\colon \bs{w}_n^{\mathsf{T}}\bs{z}\geq s\alpha_n\},\quad A(s)=sA(1),\quad f(s)=sf(1)~.
\]
Since $A(1)$ is a nonempty closed set, it can be checked that there exists $\bs{z}^\ast \in A(1)$ such that $\|\bs{z}^\ast\|=f(1)$, this vector can be calculated by solving a simple quadratic minimization problem with affine constraints. Since $\bs{w}_n^{\mathsf{T}}\bs{z}^\ast\geq \alpha_n>0$ for each $n\in[N]$, we have $\bs{z}^\ast\in \Int(C)$. Let $\bs{u}^\ast\coloneqq \frac{\bs{z}^\ast}{f(1)}\in B(\bs{0},1)\cap \Int(C)$. Then, for every $s>0$, we have $sf(1)\bs{u}^\ast \in A(s)$, whose norm yields $f(s)$.

For each $t\in [T]$, we assume that the noise vector $\bs{Y}_t$ is either $\frac14\bs{u}^\ast$ with probability $p\in (0,1)$ or it is $-\frac14\bs{u}^\ast$ with probability $1-p$, which implies $\E[\bs{Y}_t]=(\frac{p}{2}-\frac14)\bs{u}^\ast$. Moreover, due to the structure of this distribution, the difference $\hat{\bs{\mu}}_i-\bs{\mu}_i$ between the empirical and true means of a design $i\in[K]$ has a discrete distribution over a finite subset of the line segment that connects $-\frac14\bs{u}^\ast$ and $\frac14\bs{u}^\ast$, hence we may write $\hat{\bs{\mu}}_i-\bs{\mu}_i=(\frac{\hat{p}_i}{2}-\frac14) \bs{u}^\ast $, where $\hat{p}_i$ is the empirical frequency of observing $\bs{Y}_t=\frac14\bs{u}^\ast$ over all rounds $t$ at which design $i$ is sampled. 

As in \citet{auer2016pareto}, we consider the case $p_i=\frac12$ for every $i\in [K]$ as the original case, where the noise vectors are centered as in the problem setup. We will introduce modifications on $(p_i)_{i\in[N]}$ in such a way that the algorithm returns an output $P^\prime$ that is different from the original one $P$. This could be due to a change in the Pareto set or a change in the gap values of the designs.

We consider the following four possibilities for $i\in [K]$.

\emph{Case 1:} Suppose that $i\in P^\ast\setminus P$. Then, Remark 3.7 implies that there exist $j\in P$ and $\bs{u}\in B(\bs{0},1)\cap C$ such that $\bs{\mu}_j+\epsilon \bs{u}\in \bs{\mu}_i+C$. Let $J(i)\subseteq P$ be the set of all such $j$. Then, by the definition of $M(i,j)$, we have $0<M(i,j)\leq \epsilon $ for every $j \in J(i)$. In particular, $M(i,j)\leq \epsilon$ for every $j\in J(i)$ and $\tilde{\Delta}^\epsilon_i=\epsilon$. Consider a modification of $p_i$ such that the corresponding mean vector becomes $\bs{\mu}_i^\prime:=\bs{\mu}_i+2\tilde{\Delta}^\epsilon_i f(1)\bs{u}^\ast=\bs{\mu}_i+2\epsilon f(1)\bs{u}^\ast$ while all the other mean vectors remain the same. Note that $\bs{\mu}_i^\prime -\bs{\mu}_i =2\epsilon f(1)\bs{u}^\ast\in A(2\epsilon)$, i.e., $\bs{\mu}_i^\prime \in \bs{\mu}_i+2\epsilon\bs{u}+C$ for every $\bs{u}\in B(\bs{0},1)\cap C$. We claim that $i$ is not covered by any $j\in J(i)$ in the modified case. Indeed, if we had $j\in J(i)$ and $\bs{u}\in B(\bs{0},1)\cap C$ such that $\bs{\mu}_j +\epsilon \bs{u}\in \bs{\mu}_i^\prime+C$, then the properties of $\bs{\mu}_i^\prime$ would imply that $\bs{\mu}_j +\epsilon \bs{u}\in \bs{\mu}_i+2\epsilon \bs{u}+C$, i.e., $\bs{\mu}_j\in\bs{\mu}_i +\epsilon\bs{u}+C\subseteq \bs{\mu}_i+C$ so that $\bs{\mu}_j$ dominates $\bs{\mu}_i$ and $i \notin P^\ast$, a contradiction. Hence, the claim holds. On the other hand, we have $i\in (P^{\ast})^\prime$, where $(P^{\ast})^\prime$ denotes the Pareto set in the modified case. Hence, due to Definition 3.6, we must have $i\in P^\prime$, where $P^\prime$ is the set returned by the algorithm in the modified case.

\emph{Case 2:} Suppose that $i\in P^\ast\cap P$. Then, $\Delta_i^+=M(i,j)$ for some $j\in P^\ast\setminus \{i\}$. Let us modify $\bs{\mu}_i$ as $\bs{\mu}_i^\prime:=\bs{\mu}_i-k\tilde{\Delta}^\epsilon_i \bs{u}^\ast$, where $k>0$ is to be determined. We want to choose $k$ in such a way that $m^\prime(i,j)\geq 2\epsilon$ for the gap $m^\prime(i,j)$ in the modified case so that $i\notin P^\prime$ due to Definition 3.6. Note that $m^\prime(i,j)\geq 2\epsilon$ if and only if $\bs{\mu}^\prime +2\epsilon\bs{u}\in \bs{\mu}_j-C$ for every $\bs{u}\in B(\bs{0},1)\cap C$. Hence, the minimum value of $k$ is given by
\begin{align*}
	k_i&\coloneqq \inf\{k\geq 0\mid \forall \bs{u}\in B(\bs{0},1)\cap C\colon \bs{\mu}_i -k\tilde{\Delta}^\epsilon_i \bs{u}^\ast +2\epsilon \bs{u}\in \bs{\mu}_j-C\}\\
	&=\inf\{k\geq 0\mid  \forall \bs{u}\in B(\bs{0},1)\cap C\colon \bs{\mu}_j - \bs{\mu}_i + k\tilde{\Delta}^\epsilon_i \bs{u}^\ast \in 2\epsilon \bs{u}+C \}\\
	&=\inf\{k\geq 0\mid  \bs{\mu}_j - \bs{\mu}_i + k\tilde{\Delta}^\epsilon_i \bs{u}^\ast \in A(2\epsilon)\}\\
	&= \inf\{k\geq 0 \mid \forall n\in[N]\colon \bs{w}_n^{\mathsf{T}}(\bs{\mu}_j-\bs{\mu}_i+k\tilde{\Delta}^\epsilon_i \bs{u}^\ast ) \geq 2\epsilon\alpha_n\}\\
	&=\inf\Big\{k\geq 0 \mid \forall n\in [N]\colon k\geq \frac{2\epsilon\alpha_n-\bs{w}_n^{\mathsf{T}}(\bs{\mu}_j-\bs{\mu}_i)}{\tilde{\Delta}_i^\epsilon \bs{w}_n^{\mathsf{T}}\bs{u}^\ast}\Big\}\\
	&=\max_{n\in[N]}\frac{(2\epsilon\alpha_n-\bs{w}_n^{\mathsf{T}}(\bs{\mu}_j-\bs{\mu}_i))^+}{\tilde{\Delta}_i^\epsilon \bs{w}_n^{\mathsf{T}}\bs{u}^\ast}\leq \max_{n\in[N]}\frac{2\epsilon\alpha_n+(\bs{w}_n^{\mathsf{T}}(\bs{\mu}_i-\bs{\mu}_j))^+}{\tilde{\Delta}_i^\epsilon \bs{w}_n^{\mathsf{T}}\bs{u}^\ast}~.
\end{align*}
We claim that $\max_{n\in[N]}\frac{(\bs{w}_n^{\mathsf{T}}(\bs{\mu}_i-\bs{\mu}_j))^+}{\alpha_n}\leq M(i,j)$. To see this, we observe that, given a number $s\geq 0$, a sufficient condition for $s\leq M(i,j)$ to hold is given by the following:
\begin{align*}
	&\forall \bs{u}\in B(\bs{0},1)\cap C\colon \bs{\mu}_j +su \notin \bs{\mu}_i + \Int(C)\\
	&\Leftrightarrow \forall \bs{u}\in B(\bs{0},1)\cap C\ \exists n\in[N]\colon \bs{w}_n^{\mathsf{T}}(\bs{\mu}_j+s\bs{u})\leq \bs{w}_n^{\mathsf{T}}\bs{\mu}_i\\
	&\Leftrightarrow \forall \bs{u}\in B(\bs{0},1)\cap C\ \exists n\in[N]\colon s\leq \frac{\bs{w}_n^{\mathsf{T}}(\bs{\mu}_i-\bs{\mu}_j)}{\bs{w}_n^{\mathsf{T}}\bs{u}}\\
	&\Leftrightarrow \forall \bs{u}\in B(\bs{0},1)\cap C\colon s\leq \max_{n\in[N]}\frac{(\bs{w}_n^{\mathsf{T}}(\bs{\mu}_i-\bs{\mu}_j))^+}{\bs{w}_n^{\mathsf{T}}\bs{u}}~.
\end{align*}
Now, taking $s= \max_{n\in[N]}\frac{(\bs{w}_n^{\mathsf{T}}(\bs{\mu}_i-\bs{\mu}_j))^+}{\alpha_n}$, the last condition above is verified immediately thanks to the definition of $\alpha_n$. Hence, the claim holds. Then, we may write
\begin{align*}
	k_i &\leq \max_{n\in[N]}\frac{2\epsilon\alpha_n+(\bs{w}_n^{\mathsf{T}}(\bs{\mu}_i-\bs{\mu}_j))^+}{\tilde{\Delta}_i^\epsilon \bs{w}_n^{\mathsf{T}}\bs{u}^\ast}\\
	&\leq \max_{n\in[N]}\frac{2\epsilon\alpha_n + M(i,j)\alpha_n}{\tilde{\Delta}_i^\epsilon \bs{w}_n^{\mathsf{T}}\bs{u}^\ast}\leq  \max_{n\in[N]}\frac{2\tilde{\Delta}^\epsilon_i\alpha_n + \tilde{\Delta}^\epsilon_i\alpha_n}{\tilde{\Delta}_i^\epsilon \bs{w}_n^{\mathsf{T}}\bs{u}^\ast}=\max_{n\in[N]}\frac{3\alpha_n}{\bs{w}_n^{\mathsf{T}}\bs{u}^\ast}~.
\end{align*}
Therefore, choosing $k:= \max_{n\in[N]}\frac{3\alpha_n}{\bs{w}_n^{\mathsf{T}}\bs{u}^\ast}$ ensures that $m^\prime(i,j)\geq 2\epsilon$.

\emph{Case 3:} Suppose that $i\in P\setminus P^\ast$. Then, there exists $j\in P^\ast$ such that $\Delta_i^\ast = m(i,j)$ and we have $m(i,j)\leq \epsilon$ by Definition 3.6. In particular, $\tilde{\Delta}^\epsilon_i=\epsilon$. Let us modify $\bs{\mu}^\prime$ as $\bs{\mu}^\prime_i := \bs{\mu}_i -k\tilde{\Delta}^\epsilon_i\bs{u}^\ast$, where $k>0$ is to be chosen such that $m^\prime(i,j)\geq 2\epsilon$. Once this is achieved, we have $i\notin P^\prime$ according to Definition 3.6. By the proof of Proposition 3.2(iv), we have $m^\prime(i,j)\geq 2\epsilon$ if and only if $\bs{\mu}_i^\prime + 2\epsilon\bs{u}\in \bs{\mu}_j-C$ for every $\bs{u}\in B(\bs{0},1)\cap C$. Hence, the minimum value of $k$ is given by
\begin{align*}
	k_i&\coloneqq \inf\{k\geq 0 \mid \forall \bs{u}\in B(\bs{0},1)\cap C\colon \bs{\mu}_i - k \tilde{\Delta}_i^\epsilon \bs{u}^\ast +2 \epsilon \bs{u}\in \bs{\mu}_j-C\}\\
	&= \inf\{k\geq 0\mid \forall \bs{u}\in B(\bs{0},1)\cap C\colon \bs{\mu}_j - \bs{\mu}_i +k\tilde{\Delta}_i^\epsilon \bs{u}^\ast \in 2\epsilon \bs{u}+C\}\\
	&=\inf\{k\geq 0 \mid \bs{\mu}_j -\bs{\mu}_i +k \tilde{\Delta}^\epsilon_i \bs{u}^\ast \in A(2\epsilon)\}\\
	&= \inf\{k\geq 0 \mid \forall n\in[N]\colon \bs{w}_n^{\mathsf{T}}(\bs{\mu}_j-\bs{\mu}_i + k \tilde{\Delta}^\epsilon_i \bs{u}^\ast)\geq 2\epsilon\alpha_n\}\\
	&= \inf\Big\{k\geq 0 \mid \forall n\in[N]\colon k\geq \frac{2\epsilon \alpha_n - \bs{w}_n^{\mathsf{T}}(\bs{\mu}_j - \bs{\mu}_i)}{\tilde{\Delta}^\epsilon_i \bs{w}_n^{\mathsf{T}}\bs{u}^\ast}\Big\}\\
	&=\max_{n\in [N]}\frac{(2\epsilon\alpha_n - \bs{w}_n^{\mathsf{T}}(\bs{\mu}_j-\bs{\mu}_i))^+}{\tilde{\Delta}^\epsilon_i \bs{w}_n^{\mathsf{T}}\bs{u}^\ast}\\
	&\leq \max_{n\in [N]}\frac{(2\epsilon\alpha_n - m(i,j)\alpha_n)^+}{\tilde{\Delta}^\epsilon_i \bs{w}_n^{\mathsf{T}}\bs{u}^\ast}=\max_{n\in [N]}\frac{2\epsilon\alpha_n - m(i,j)\alpha_n}{\tilde{\Delta}^\epsilon_i \bs{w}_n^{\mathsf{T}}\bs{u}^\ast}\leq \max_{n\in[N]}\frac{2\alpha_n}{\bs{w}_n^{\mathsf{T}}\bs{u}^\ast}~.
\end{align*}
Hence, choosing $k:=\max_{n\in[N]}\frac{2\alpha_n}{\bs{w}_n^{\mathsf{T}}\bs{u}^\ast}$ ensures that $m^\prime(i,j)\geq 2\epsilon$.

\emph{Case 4:} Suppose that $i\notin P^\ast\cup P$. Let us modify $\bs{\mu}^\prime$ as $\bs{\mu}^\prime_i := \bs{\mu}_i +k\tilde{\Delta}^\epsilon_i\bs{u}^\ast$, where $k>0$ is to be chosen such that $i$ is not covered by any $j\in P^\ast$ in the modified case. For such $k$, $i$ will automatically be in the Pareto set of the modified case, i.e., $i\in (P^\ast)^\prime$, hence we must have $i\in P^\prime$ due to Definition 3.6. To determine the desired value of $k$, note that the largest value of $k$ for which $i$ is covered by $j\in P^\ast$ in the modified case is given by
\begin{align*}
	k_{ij}&\coloneqq \sup\{k\geq 0\mid \exists \bs{u}\in B(\bs{0},1)\cap C\colon \bs{\mu}_i^\prime \in \bs{\mu}_j+\epsilon\bs{u}-C\}\\
	&=\sup_{\bs{u}\in B(\bs{0},1)\cap C}\sup\{k\geq 0 \mid \forall n\in [N]\colon \bs{w}_n^{\mathsf{T}}\bs{\mu}_i + k \tilde{\Delta}^\epsilon_i \bs{w}^{\mathsf{T}}_n \bs{u}^\ast \leq \bs{w}_n^{\mathsf{T}}\bs{\mu}_j + \epsilon \bs{w}_n^{\mathsf{T}}\bs{u}\}\\
	&=\sup_{\bs{u}\in B(\bs{0},1)\cap C}\sup\Big\{k\geq 0\mid \forall n\in [N]\colon k\leq \frac{\bs{w}_n^{\mathsf{T}}(\bs{\mu}_j-\bs{\mu}_i+\epsilon \bs{u})}{\tilde{\Delta}^\epsilon_i \bs{w}_n^{\mathsf{T}}\bs{u}^\ast}
	\Big \}\\
	&=\sup_{\bs{u}\in B(\bs{0},1)\cap C}\min_{n\in[N]} \frac{(\bs{w}_n^{\mathsf{T}}(\bs{\mu}_j-\bs{\mu}_i+\epsilon \bs{u}))^+}{\tilde{\Delta}^\epsilon_i \bs{w}_n^{\mathsf{T}}\bs{u}^\ast}\leq 
	\sup_{\bs{u}\in B(\bs{0},1)\cap C}\min_{n\in[N]} \frac{(\bs{w}_n^{\mathsf{T}}(\bs{\mu}_j-\bs{\mu}_i))^+ + \epsilon \bs{w}_n^{\mathsf{T}}\bs{u}}{\tilde{\Delta}^\epsilon_i \bs{w}_n^{\mathsf{T}}\bs{u}^\ast}\\
	&\leq \sup_{\bs{u}\in B(\bs{0},1)\cap C}\min_{n\in[N]}\frac{m(i,j)\alpha_n + \epsilon \bs{w}_n^{\mathsf{T}}\bs{u}}{\tilde{\Delta}_i^\epsilon \bs{w}_n^{\mathsf{T}}\bs{u}^\ast}\leq 
	\sup_{\bs{u}\in B(\bs{0},1)\cap C}\min_{n\in [N]}\frac{\alpha_n+\bs{w}_n^{\mathsf{T}}\bs{u}}{\bs{w}_n^{\mathsf{T}}\bs{u}^\ast}\\
	& \leq \sup_{\bs{u}\in B(\bs{0},1)\cap C}\max_{n\in [N]}\frac{\alpha_n+\bs{w}_n^{\mathsf{T}}\bs{u}}{\bs{w}_n^{\mathsf{T}}\bs{u}^\ast}\leq  \max_{n\in[N]}\frac{2\alpha_n}{\bs{w}_n^{\mathsf{T}}\bs{u}^\ast}~,
\end{align*}
where the passage to the penultimate line follows from Proposition 3.2(iv). Hence, choosing $k:=1+\max_{n\in[N]}\frac{2\alpha_n}{\bs{w}_n^{\mathsf{T}}\bs{u}^\ast}$ ensures that $i$ is not covered by any $j\in P^\ast$.

Since we are able to change the returned output in each possibility, as argued in \cite{auer2016pareto}, the algorithm has to sample design $i$ at least $\Omega(\frac{1}{(\tilde{\Delta}^\epsilon_i)^2}\log(\frac{1}{\delta}))$ times in order to detect these changes, which completes the proof.

\subsection{Proof of Theorem \ref{lowerboundthm_new}}

Let $\bs{v}\in \R^D$ be such that $\bs{v}\notin C\cup (-C)$. Note that $d(\bs{v},(\bs{v}+C)\cap C)>0$. The arguments in the proof of Proposition \ref{lem:beta1lambda} show that the function $\bs{x}\mapsto d(\bs{x},(\bs{x}+C)\cap C)$ is positively homogeneous. Hence, by scaling, we may assume without loss of generality that 
\[
d(\bs{v},(\bs{v}+C)\cap C)=\epsilon~.
\]
Let $\bs{\mu}\in C\setminus\{\bs{0}\}$.
We set the mean reward vectors as
\[
\bs{\mu}_1=\bs{0},\quad \bs{\mu}_2=\bs{\mu},\quad \bs{\mu}_i = \bs{\mu}-\bs{v}\quad \forall i\in [K]\setminus[2]~.
\]
Since $\bs{v}\notin C\cup (-C)$, it follows that $P^\ast=[K]\setminus[1]$.

Let $\bs{u}^\ast\in \Int(C)$ be such that
\[
\beta_2=\frac{d(\bs{u}^\ast,(\Int(C))^c\cap(\bs{u}^\ast-C))}{d(\bs{u}^\ast,(\Int(C))^c)}~.
\]
It can be checked that, for every $a>0$, replacing $\bs{u}^\ast$ with $a\bs{u}^\ast$ does not change the value of the above ratio; hence, we assume that $d(\bs{u}^\ast,(\Int(C))^c)=1$ without loss of generality (otherwise, we may take $a=1/d(\bs{u}^\ast,(\Int(C))^c)$). By Lemma \ref{lem:beta2}, we have
\begin{equation}\label{beta2num}
	d(\bs{u}^\ast,(\Int(C))^c\cap(\bs{u}^\ast-C))=\min_{n\in [N]}\frac{\bs{w}_n^{\mathsf{T}}\bs{u}^\ast}{\alpha_n}~.
\end{equation}

For each $t\in [T]$, we assume that the noise vector $\bs{Y}_{i,t}$ is either $\frac14\bs{u}^\ast$ with probability $p_i = 1/2$ or it is $-\frac14\bs{u}^\ast$ with probability $1-p_i$, which implies that $\E[\bs{Y}_{i,t}]=\bs{0}$. Moreover, due to the structure of this distribution, the difference $\hat{\bs{\mu}}_i-\bs{\mu}_i$ between the empirical and true means of a design $i$ has a discrete distribution over a finite subset of the line segment that connects $-\frac14\bs{u}^\ast$ and $\frac14\bs{u}^\ast$, hence we may write $\hat{\bs{\mu}}_i-\bs{\mu}_i=(\frac{\hat{p}_i}{2}-\frac14) \bs{u}^\ast$, where $\hat{p}_i$ is the empirical frequency of observing $\bs{Y}_{i,t}=\frac14\bs{u}^\ast$ over all rounds $t$ at which design $i$ is sampled. 

We proceed by choosing an arbitrary design $i \in [K] \setminus [2]$. Under the original problem $(\bs{\mu}_j)_{j\in[K]}$, we have $i \in P^*$. Then, we define two modified problems $(\bs{\mu}^\prime_j)_{j\in[K]}$ and $(\bs{\mu}^{\prime\prime}_j)_{j\in[K]}$, where $\bs{\mu}_i$ is replaced by $\bs{\mu}^\prime_i$ and $\bs{\mu}^{\prime\prime}_i$ respectively, while other designs' mean reward vectors remain unchanged, i.e., $\bs{\mu}_j=\bs{\mu}_j^\prime=\bs{\mu}_j^{\prime\prime}$ for each $j\neq i$. The new mean reward vectors will be formed by only changing the parameter of noise from $p_i$ to $p'_i$ and $p''_i$ whose values will be specified later. We will carefully set the noise direction to make the Pareto set identification problem difficult. Then, we will argue that any $i \notin P$ should be in $P'$, the returned set of the modified problem $(\bs{\mu}^\prime_j)_{j\in[K]}$. We will also argue that any $i \in P$ should not be in $P''$, the returned set of the modified problem $(\bs{\mu}''_j)_{j\in[K]}$. Making an analogy with the coin bias problem as in \cite{auer2016pareto}, we will conclude that design $i$ needs to be sampled $\Omega (\frac{\beta_2^2}{\epsilon^2} \log \frac{1}{\delta})$ times in order to distinguish these cases with at least $1-\delta$ probability. Since the number of such designs is $K-2$, this yields $\Omega (\frac{K \beta_2^2}{\epsilon^2} \log \frac{1}{\delta})$ lower bound on the regret.

\emph{Case 1:} Fix some $i>2$ and suppose that $i\notin P$. Then, Remark \ref{rem:PAC} implies that there exist $j \in P$ and $\bs{u}\in B(\bs{0},1)\cap C$ such that $\bs{\mu}_j+\epsilon \bs{u}\in \bs{\mu}_i+C$. Let $J(i)\subseteq P$ be the set of all such $j$. By the definition of $M(i,j)$, we have $0<M(i,j)\leq \epsilon $ for every $j \in J(i)$. In particular, $\tilde{\Delta}^\epsilon_i=\epsilon$. 

Let us consider the modification $\bs{\mu}^\prime_i=\bs{\mu}_i+k\epsilon\bs{u}^\ast$, where $k>0$ is to be chosen such that $i$ is not covered by any $j\in P^\ast$ in the modified case. For such $k$, $i$ will automatically be in the Pareto set of the modified case, i.e., $i\in (P^\ast)^\prime$, hence we must have $i\in P^\prime$ due to Definition \ref{defn:PAC}. To determine the desired value of $k$, note that the largest value of $k$ for which $i$ is covered by $j\in P^\ast$ in the modified case is given by
\[
k_{ij}\coloneqq \sup\{k\geq 0\mid \exists \bs{u}\in B(\bs{0},1)\cap C\colon \bs{\mu}^\prime_i \in \bs{\mu}_j+\epsilon\bs{u}-C\}\leq  \max_{n\in[N]}\frac{2\alpha_n}{\bs{w}_n^{\mathsf{T}}\bs{u}^\ast}~.
\]
as computed in Case 4 of the proof of Theorem \ref{lowerboundthm}. Moreover, by \eqref{beta2num} and the definition of $\bs{u}^\ast$, we have
\begin{equation}\label{eq:reciprocal}
	\max_{n\in[N]}\frac{2\alpha_n}{\bs{w}_n^{\mathsf{T}}\bs{u}^\ast}=\frac{2}{d(\bs{u}^\ast,(\Int(C))^c\cap (\bs{u}^\ast-C))}=\frac{2}{\beta_2}\frac{1}{d(\bs{u}^\ast,(\Int(C))^c)}= \frac{2}{\beta_2}~.
\end{equation}
Hence, choosing any $k$ that satisfies
\[
k>\frac{2}{\beta_2}
\]
does the job. For such $k$, we automatically have $i\in (P^\ast)^\prime$, which implies that $i\in P^\prime$ due to Definition \ref{defn:PAC}.

\emph{Case 2:} Fix $i>2$ and suppose that $i\in P$. Let us consider the modification $\bs{\mu}^{\prime\prime}_i=\bs{\mu}_i-k\epsilon\bs{u}^\ast$, where $k>0$ is to be chosen such that $m^{\prime\prime}(i,2)\geq 2\epsilon$. Hence, for such $k$, we have $i\notin P^{\prime\prime}$ in the modified case thanks to Definition \ref{defn:PAC}.

Following a similar derivation as in Case 2 of the proof of Theorem \ref{lowerboundthm}, the minimum $k$ for which $m^{\prime\prime}(i,2)\geq 2\epsilon$ is given by
\begin{align*}
	k_i\coloneqq \inf\{k\geq 0\mid \forall \bs{u}\in B(\bs{0},1)\cap C\colon \bs{\mu}_i -k\epsilon \bs{u}^\ast +2\epsilon \bs{u}\in \bs{\mu}_2-C\}
	\leq \max_{n\in[N]}\frac{2\epsilon\alpha_n + M(i,2)\alpha_n}{\epsilon \bs{w}_n^{\mathsf{T}}\bs{u}^\ast}~.
\end{align*}
Note that $M(i,2)=d(\bs{\mu}_2-\bs{\mu}_i, (\bs{\mu}_2-\bs{\mu}_i+C)\cap C)=d(\bs{v}, (\bs{v}+C)\cap C)=\epsilon$ by construction. Hence,
\[
k_i\leq \max_{n\in[N]}\frac{2\epsilon\alpha_n + \epsilon\alpha_n}{\epsilon \bs{w}_n^{\mathsf{T}}\bs{u}^\ast}= \max_{n\in[N]}\frac{3\alpha_n}{\bs{w}_n^{\mathsf{T}}\bs{u}^\ast}=\frac{3}{\beta_2}~,
\]
where the last equality is by \eqref{eq:reciprocal}. Hence, we may simply choose $k=\frac{3}{\beta_2}$ and ensure that $i\notin P^{\prime\prime}$.

\com{Some useful identities:
	\begin{align*}
		& \sin(x)-\cos(x) = -\sqrt{2} \sin\left(\frac{\pi}{4}-x\right) \\
		& \sin(x+y)\cos(x-y) - \cos(x+y)\sin(x-y) = \sin(2y) \\
		& \sin(x-y)\sin(x+y) + \cos(x-y)\cos(x+y) = \cos(2y) \\
		& 1+\tan^2x = \sec^2 x = 1/\cos^2 x \\
		& a \sin x + b \cos x = \sqrt{a^2 + b^2} \sin(x+y), \text{ where} \tan y = b/a \\
		& \sqrt{2} \sin(x/2) \leq \sin(x) \leq 2 \sin(x/2) \text{ for } x \in [0, \pi/2] \\
		& \bs{w}_1^T \bs{x} = \frac{1}{\sqrt{2}} \left[ -\sin \left(\frac{\pi}{4} - \frac{\theta}{2} \right) + \cos \left(\frac{\pi}{4} - \frac{\theta}{2} \right) \right] 
		= \frac{1}{\sqrt{2}} \left[ \sqrt{2} \sin \left(\frac{\pi}{4} - \frac{\pi}{4} + \frac{\theta}{2} \right)    \right]
		= \sin \left( \frac{\theta}{2} \right) \\
		& \bs{w}_2^T \bs{x} = \frac{1}{\sqrt{2}} \left[ \sin \left(\frac{\pi}{4} + \frac{\theta}{2} \right) - \cos \left(\frac{\pi}{4} + \frac{\theta}{2} \right) \right] 
		= \frac{1}{\sqrt{2}} \left[ -\sqrt{2} \sin \left(\frac{\pi}{4} - \frac{\pi}{4} - \frac{\theta}{2} \right)    \right]
		=- \sin \left( -\frac{\theta}{2} \right) = \sin \left( \frac{\theta}{2} \right) \\
		& \bs{w}_1^T \bs{u}_l = 0 \\
		& \bs{w}_2^T \bs{u}_l = \sin \left(\frac{\pi}{4} + \frac{\theta}{2} \right) \cos \left(\frac{\pi}{4} - \frac{\theta}{2} \right)     -  \cos \left(\frac{\pi}{4} + \frac{\theta}{2} \right) \sin \left(\frac{\pi}{4} - \frac{\theta}{2} \right)
		= \sin(\theta) \\
		& \bs{w}_1^T \bs{u}_u = \bs{w}_2^T \bs{u}_l = \sin(\theta) \\
		& \bs{w}_2^T \bs{u}_u = 0 \\
		& \bs{w}_1^T \bs{w}_2 = 
		- \left[  \sin \left(\frac{\pi}{4} - \frac{\theta}{2} \right) \sin \left(\frac{\pi}{4} + \frac{\theta}{2} \right) +  \cos \left(\frac{\pi}{4} - \frac{\theta}{2} \right) \cos \left(\frac{\pi}{4} + \frac{\theta}{2} \right) \right] = -\cos(\theta)
	\end{align*}
}

						\subsection{Proof of Theorem \ref{thm:samplecomp}}
						
						Given $i,j\in[K]$, let $\hat{m}(i,j) := d(\hat{\bs{\mu}}_j - \hat{\bs{\mu}}_i, (\Int(C))^c\cap (\hat{\bs{\mu}}_j-\hat{\bs{\mu}}_i-C))$ and $\hat{M}(i,j) := d(\hat{\bs{\mu}}_j - \hat{\bs{\mu}}_i, C\cap (\hat{\bs{\mu}}_j-\hat{\bs{\mu}}_i+C))$ represent the empirical estimates of $m(i,j)$ and $M(i,j)$, respectively. In the next lemma, we prove that the following conditions in terms of the gaps are sufficient for $(\epsilon,\delta)$-PAC Pareto set identification.
						
						\textbf{Condition 1.} \emph{For every $i \in P^\ast$ and $j\in [K] \setminus \{ i \}$, $M(i,j)> \epsilon$ implies that $\hat{M}(i,j)>0$.}\\
						\textbf{Condition 2.} \emph{For every $i \notin P^\ast$ and $j \in P^\ast$, $m(i,j) > \epsilon$ implies that $\hat{m}(i,j) >0$.}
						
						\begin{lemma}\label{lem:sufficient}
							If Conditions 1 and 2 hold, then $P$ is an $(\epsilon,\delta)$-PAC Pareto set.
						\end{lemma}
						\begin{proof}
							Assume that Condition 1 holds. To get a contradiction, suppose that property (i) in Definition \ref{defn:PAC} does not hold. By Remark \ref{rem:PAC}, this implies that there exists $i\in P^\ast$ such that $\bs{\mu}_i\notin \bs{\mu}_j + (B(\bs{0},\epsilon)\cap C) - C$ for each $j\in P$. Hence, $\bs{\mu}_j-\bs{\mu}_i\notin (B(\bs{0},\epsilon)\cap(-C))+C$ for each $j\in P$. 
							Next, we will show that this implies $M(i,j)=d(\bs{\mu}_j-\bs{\mu}_i,C\cap(\bs{\mu}_j-\bs{\mu}_i+C))>\epsilon$ for each $j\in P$. 
							
							The result follows from a more general argument. Consider $\bs{x} \in \mathbb{R}^D$. Let $\bs{x}^* \in C \cap (\bs{x} + C)$ be such that $d(\bs{x} , C \cap (\bs{x} + C)) = d(\bs{x}, \bs{x}^*)$. Note that $\bs{x}^* \in C \cap (\bs{x} + C)$ implies the existence of unit vectors $\bs{u}_1, \bs{u}_2 \in C$ and scalars $c_1, c_2 \in \mathbb{R}_{+}$ such that $\bs{x}^* = \bs{x} + c_1 \bs{u}_1 = c_2 \bs{u}_2$. Hence, we have $\bs{x} = c_2 \bs{u_2} - c_1 \bs{u}_1$. We claim that $d(\bs{x}, \bs{x}^*) \leq \epsilon$ implies that $\bs{x} \in ( B(\bs{0}, \epsilon) \cap (-C) ) + C$. Now assume that $d(\bs{x}, \bs{x}^*) \leq \epsilon$. Note that $d(\bs{x}, \bs{x}^*) = c_1 \leq \epsilon$. It immediately follows that $\bs{x} = c_2 \bs{u_2} - c_1 \bs{u}_1 \in ( B(\bs{0}, \epsilon) \cap (-C) ) + C$, since $- c_1 \bs{u}_1 \in  ( B(\bs{0}, \epsilon) \cap (-C) )$ and $c_2 \bs{u}_2 \in C$. This proves our claim. Therefore, $\bs{x} \notin ( B(\bs{0}, \epsilon) \cap (-C) ) + C$ implies that $d(\bs{x} , C \cap (\bs{x} + C)) = d(\bs{x}, \bs{x}^*) > \epsilon$. Plugging $\bs{x} = \bs{\mu}_j-\bs{\mu}_i$ to the argument above gives the desired result.
							
							
							Then, Condition 1 implies $\hat{M}(i,j)>0$ for each $j\in P$. By Corollary \ref{cor:mM}(iii) applied to the random partial orders $\widehat{\prec}_C, \widehat{\preceq}_C$, we get $i\widehat{\npreceq}_C j$ for each $j\in P$. Hence, we must have $i\in P$. In particular, taking $j=i$ gives $\bs{\mu}_i\notin \bs{\mu}_i+(B(\bs{0},\epsilon)\cap C)-C$, a contradiction. Hence, property (i) holds.
							
							Assume that Condition 2 holds. Let $i\in P\setminus P^\ast$. Then, there is no $j\in [K]\setminus \{i\}$ such that $i\widehat{\preceq}_C j$. By Corollary \ref{cor:mM}(iii) applied to the random orders $\widehat{\prec}_C, \widehat{\preceq}_C$, we have $\hat{m}(i,j)=0$ for every $j\in [K]\setminus\{i\}$, hence, for every $j\in P^\ast$. Then, Condition 2 implies that $m(i,j)\leq \epsilon$ for every $j\in P^\ast$. Hence, $\Delta^\ast_i\leq \epsilon$, i.e., property (ii) in Definition \ref{defn:PAC} holds.
							%
							%
						\end{proof}
						
						For each $i,j\in [K]$ with $i\neq j$, let us define $\bs{\Delta}_{ij} = \bs{\mu}_j - \bs{\mu}_i$, $\hat{\bs{\Delta}}_{ij} = \hat{\bs{\mu}}_j - \hat{\bs{\mu}}_i$. The next lemma shows that the following conditions in terms of the deviation of $\hat{\bs{\Delta}}_{ij}$ from $\bs{\Delta}_{ij}$ are sufficient for $(\epsilon,\delta)$-PAC Pareto set identification.
						
						\textbf{Condition a.} \emph{For every $i\in P^\ast$ and $j\in [K]\setminus\{i\}$, $d(\bs{\Delta}_{ij},C\cap (\bs{\Delta}_{ij}+C))>\epsilon$ implies that $\|\hat{\bs{\Delta}}_{ij}-\bs{\Delta}_{ij}\|_2< d(\bs{\Delta}_{ij},C)$.}\\
						\textbf{Condition b.} \emph{For every $i\notin P^\ast$ and $j\in P^\ast$, $d(\bs{\Delta}_{ij},(\Int(C))^c\cap (\bs{\Delta}_{ij}-C))>\epsilon$ implies that $\|\hat{\bs{\Delta}}_{ij}-\bs{\Delta}_{ij}\|_2< d(\bs{\Delta}_{ij},(\Int(C))^c)$.}
						
						\begin{lemma}\label{lem:identification}
							If Conditions a and b hold, then $P$ satisfies the success condition in Definition \ref{defn:PAC}. 
						\end{lemma}
						
						\begin{proof}
							Assume that Condition a holds. We verify Condition 1. Let $i\in P^\ast$ and $j\in [K]\setminus\{i\}$ such that $M(i,j)>\epsilon$, that is, $d(\bs{\Delta}_{ij},C\cap(\bs{\Delta}_{ij}+C))>\epsilon$ by Proposition \ref{prop:M}(ii). Then, Condition a implies that $\|\hat{\bs{\Delta}}_{ij}-\bs{\Delta}_{ij}\|_2< d(\bs{\Delta}_{ij},C)$. Since $d(\bs{\Delta}_{ij},C\cap (\bs{\Delta}_{ij}+C))>0$ and $C\cap (\bs{\Delta}_{ij}+C)$ is a closed set, we have $\bs{\Delta}_{ij}\notin C\cap (\bs{\Delta}_{ij}+C)$. Hence, $\bs{\Delta}_{ij}\notin C$. We claim that $\hat{\bs{\Delta}}_{ij}\notin C$. To get a contradiction, suppose that $\hat{\bs{\Delta}}_{ij}\in C$. Then, we have $\|\hat{\bs{\Delta}}_{ij}-\bs{\Delta}_{ij}\|_2\geq \inf_{\bs{c}\in C}\|\bs{\Delta}_{ij}-\bs{c}\|_2= d(\bs{\Delta}_{ij},C)$, which is a contradiction. Hence, the claim holds and we have $\hat{\bs{\Delta}}_{ij}\notin C\cap (\hat{\bs{\Delta}}_{ij}+C)$. By Proposition \ref{prop:M}(ii), we obtain $\hat{M}(i,j)=d(\hat{\bs{\Delta}}_{ij},C\cap(\hat{\bs{\Delta}}_{ij}+C))>0$. Hence, Condition 1 holds.
							
							
							Assume that Condition b holds. We verify Condition 2. Let $i\notin P^\ast$ and  $j\in P^\ast$ such that $m(i,j)>\epsilon$, that is, $d(\bs{\Delta}_{ij}, (\Int(C))^c\cap(\bs{\Delta}_{ij}-C))>\epsilon$ by Proposition \ref{prop:m}(ii). Then, Condition b implies that $\|\hat{\bs{\Delta}}_{ij}-\bs{\Delta}_{ij}\|_2< d(\bs{\Delta}_{ij},(\Int(C))^c)$. Moreover, since $(\Int(C))^c\cap(\bs{\Delta}_{ij}-C)$ is a closed set, we have $\bs{\Delta}_{ij}\notin (\Int(C))^c\cap(\bs{\Delta}_{ij}-C)$ and hence $\bs{\Delta}_{ij}\in \Int(C)$. We claim that $\hat{\bs{\Delta}}_{ij}\in \Int(C)$. Supposing otherwise that $\hat{\bs{\Delta}}_{ij}\in (\Int(C))^c$, we get $\|\hat{\bs{\Delta}}_{ij}-\bs{\Delta}_{ij}\|_2\geq \inf_{\bs{c}\in (\Int(C))^c}\| \bs{\Delta}_{ij}-\bs{c}\|_2=d(\bs{\Delta}_{ij},(\Int(C))^c)$, which is a contradiction. Hence, the claim holds and we have $\hat{\bs{\Delta}}_{ij}\notin (\Int(C))^c\cap(\hat{\bs{\Delta}}_{ij}-C)$. By Proposition \ref{prop:m}(ii), we get $\hat{m}(i,j)=d(\hat{\bs{\Delta}}_{ij},C\cap (\hat{\bs{\Delta}}_{ij}+C))>0$. Hence, Condition 2 holds.
						\end{proof}
						
						
						Let $i,j\in [K]$ with $i\neq j$. The next lemma explains how $\hat{\bs{\Delta}}_{ij}$ concentrates around $\bs{\Delta}_{ij}$ in $\ell_2$ norm as a function of the exploration parameter $L$. Let us introduce the constant
						\[
						\theta_{ij}:=\begin{cases} \frac{d(\bs{\Delta}_{ij},C)}{d(\bs{\Delta}_{ij},C\cap(\bs{\Delta}_{ij}+C))} & \text{if }\bs{\Delta}_{ij}\notin C~,\\  \frac{d(\bs{\Delta}_{ij},(\Int(C))^c)}{d(\bs{\Delta}_{ij},(\Int(C))^c \cap(\bs{\Delta}_{ij}-C))} & \text{if }\bs{\Delta}_{ij}\in \Int(C)~,\\
							1 & \text{if } \bs{\Delta}_{ij}\in\bd(C)~.\end{cases}
						\]
						Moreover, for a given $(\epsilon,\delta)$, we define the per-design sampling budget as
						\begin{align*}
							g(\epsilon,\delta) :=  \Big\lceil \frac{4 \beta^2 c^2 \sigma^2}{\epsilon^2} \log \Big(\frac{4 D}{\delta}\Big) \Big\rceil~.
						\end{align*}

						\begin{lemma}\label{lem:deltaconcentrate}
							Suppose that $L= g(\epsilon,\delta)$ noisy observations are evaluated for each design, where $c>0$ is a constant. Then, there exists a choice of $c>0$ (free of all problem parameters) such that, for each $i,j\in[K]$ with $i\neq j$,  we have $||   \hat{\bs{\Delta}}_{ij} - \bs{\Delta}_{ij}   ||_2 \leq \epsilon \theta_{ij}$ with probability at least $1-\delta$. 
						\end{lemma}
						
						\begin{proof}
							By \citealt[Corollary 7]{jin2019short}, there exists an absolute constant $c>0$ such that, with probability at least $1-\delta/2$, we have
							\begin{align*}
								&L  \| \hat{\bs{\mu}}_i -  \bs{\mu}_i \|_2 
								= L  \Big|\Big| \frac{ \sum_{t=1}^{LK} \bs{X}_t \mathbb{I}(I_t = i)}{L}-  \bs{\mu}_i  \Big|\Big|_2  \\
								& = \Big|\Big|  \sum_{t=1}^{K L}  ( \bs{X}_t -  \bs{\mu}_i )\mathbb{I}(I_t = i) \Big|\Big|_2  
								\leq c  \sqrt{ L \sigma^2 \log \Big(\frac{4 D}{\delta}\Big) } ~,
							\end{align*}
							that is, $|| \hat{\bs{\mu}}_i -  \bs{\mu}_i ||_2 \leq c \sqrt{\frac{\sigma^2}{L} \log (\frac{4 D}{\delta})}$. Setting $L= \lceil (4 \beta^2 c^2 \sigma^2 / \epsilon^2) \log (4 D / \delta) \rceil$ ensures that $|| \hat{\bs{\mu}}_i -  \bs{\mu}_i ||_2 \leq \epsilon/(2\beta)$ with probability at least $1-\delta/2$.
							Applying a union bound, we obtain
							$\Pr\{ || \hat{\bs{\mu}}_i -  \bs{\mu}_i ||_2 \leq \frac{\epsilon}{2\beta} , || \hat{\bs{\mu}}_j -  \bs{\mu}_j ||_2 \leq \frac{\epsilon}{2\beta} \}\geq 1-\delta$.
							Noting that $|| \hat{\bs{\Delta}}_{ij} - \bs{\Delta}_{ij}   ||_2 = ||  \hat{\bs{\mu}}_j-  \bs{\mu}_j + \bs{\mu}_i - \hat{\bs{\mu}}_i ||_2 \leq || \hat{\bs{\mu}}_i -  \bs{\mu}_i ||_2 + || \hat{\bs{\mu}}_j -  \bs{\mu}_j ||_2$, we have $\Pr\{\|\hat{\bs{\Delta}}_{ij} - \bs{\Delta}_{ij} \|_2\leq \epsilon/\beta \}\geq 1-\delta$. If $\bs{\Delta}_{ij}\notin C$, then
							\begin{align*}
								\frac{\epsilon}{\beta}\leq \frac{\epsilon}{\beta_1}\leq \frac{\epsilon d(\bs{\Delta}_{ij},C)}{d(\bs{\Delta}_{ij},C\cap (\bs{\Delta}_{ij}+C))}=\epsilon\theta_{ij}~;
							\end{align*}
							if $\bs{\Delta}_{ij}\in \Int(C)$, then
							\begin{align*}
								\frac{\epsilon}{\beta}\leq \frac{\epsilon}{\beta_2}\leq \frac{\epsilon d(\bs{\Delta}_{ij},(\Int(C))^c)}{d(\bs{\Delta}_{ij},(\Int(C))^c\cap (\bs{\Delta}_{ij}-C))}~;
							\end{align*}
							and if $\bs{\Delta}_{ij}\in \bd(C)$, then $\epsilon/\beta\leq \epsilon =\epsilon\theta_{ij}$; see \eqref{eq:Csup2} for the definitions of $\beta_1, \beta_2$. Hence, the result follows.
						\end{proof}
						
						After defining the lemmas that will be used in the proof, we complete the proof of Theorem \ref{thm:samplecomp} by the argument below. 
						
						With the given choice of $L$, by Lemma \ref{lem:deltaconcentrate}, for every $i,j \in [K]$ with $i \neq j$, we have $||   \hat{\bs{\Delta}}_{ij} - \bs{\Delta}_{ij}   ||_2 \leq \epsilon\theta_{ij}$ with probability at least $1 - 2\delta/(K (K-1))$. The application of union bound shows that with probability at least $1-\delta$, $|| \hat{\bs{\Delta}}_{ij} - \bs{\Delta}_{ij} ||_2 \leq \epsilon\theta_{ij}$ simultaneously for all $i,j \in [K]$ such that $i \neq j$. Under this event, we verify Conditions a and b. Let $i\in P^\ast$ and $j\in [K]\setminus\{i\}$ such that $d(\bs{\Delta}_{ij},C\cap(\bs{\Delta}_{ij}+C))>\epsilon$. In particular, $\bs{\Delta}_{ij}\notin C$. Hence, we have
						\[
							\| \hat{\bs{\Delta}}_{ij} - \bs{\Delta}_{ij} \|_2 \leq \epsilon\theta_{ij}=\frac{\epsilon d(\bs{\Delta}_{ij},C)}{d(\bs{\Delta}_{ij},C\cap(\bs{\Delta}_{ij}+C))}\\
							<d(\bs{\Delta}_{ij},C)~.
						\]
						This shows that Condition a holds. Let $i\notin P^\ast$ and $j\in P^\ast$ such that $d(\bs{\Delta}_{ij},(\Int(C))^c\cap (\bs{\Delta}_{ij}-C))>\epsilon$. In particular, $\bs{\Delta}_{ij}\in \Int(C)$. Hence, we have
						\[
							\| \hat{\bs{\Delta}}_{ij} - \bs{\Delta}_{ij} \|_2 \leq \epsilon\theta_{ij}= \frac{\epsilon d(\bs{\Delta}_{ij},(\Int(C))^c)}{d(\bs{\Delta}_{ij}, (\Int(C))^c\cap(\bs{\Delta}_{ij}-C))}\\
							<d(\bs{\Delta}_{ij},(\Int(C))^c)~.
						\]
						This shows that Condition b holds. Therefore, by Lemma \ref{lem:identification}, $P$ is an $(\epsilon,\delta)$-PAC Pareto set.
						
						When $D=1$, the sample complexity of na\"{i}ve elimination will match the one given in \citet{even2006action}. In order to match, the analysis in Lemma \ref{lem:sufficient} and Theorem \ref{thm:samplecomp} needs to be updated. In particular, we no longer need concentration of $\hat{\Delta}_{ij}$s. Showing concentration of $\hat{\mu}_j$s will suffice.
						
						
						\section{SUPPLEMENTAL NUMERICAL RESULTS}
						
						The following table contains additional numerical results about the experiments performed in Section \ref{sec:experiment}. Code for the paper is included as non-textual supplementary material (ZIP file). 
						
						\begin{table}[h!]
							\centering
							\caption{Additional results for experiments conducted in Section \ref{sec:experiment}. $\theta$: cone angle in degrees. SR$_{\theta}$: success rate (in \%) for $C_{\theta}$. NF1$_{\theta}$: average number of designs in $P^*_{\theta}$ that fail success condition (i) in Definition \ref{defn:PAC}. NF2$_{\theta}$: average number of designs in $P \setminus P^*_{\theta}$ that fail success condition (ii) in Definition \ref{defn:PAC}. PM$_{\theta}$: $(|P^*_{\theta} \setminus P|/|P^*_{\theta}|) \times 100$.}
							\label{tbl:ressults_nomultirow}
							\setlength{\tabcolsep}{2pt}
							\begin{tabular}{llrrrrrrrrrrrr}
								\toprule
								$L$ & $\epsilon$  &  SR$_{45}$ &  SR$_{90}$ &  SR$_{135}$ &  NF1$_{45}$  &  NF1$_{90}$  &  NF1$_{135}$  &  NF2$_{45}$  &  NF2$_{90}$  &  NF2$_{135}$  &  PM$_{45}$  &  PM$_{90}$  &  PM$_{135}$  \\
								\midrule
								$10^2$    & $10^{-3}$  &   0 &   0 &    0 &   12.44 &    7.61 &     1.83 &    9.26 &    7.07 &     2.52 &   24.8 &   29.3 &  18.3 \\
								& $10^{-2}$  &   0 &   0 &    0 &   11.52 &    6.87 &     1.83 &    7.44 &    6.08 &     2.52 &   24.8 &   29.3 &    18.3 \\
								& $10^{-1}$ &   1 &   7 &   27 &    4.22 &    1.21 &     0.63 &    1.27 &    0.98 &     0.83 &   24.8 &   29.3 &    18.3 \\
								$10^3$   & $10^{-3}$  &   0 &   0 &   29 &    4.45 &    3.33 &     0.32 &    5.78 &    3.72 &     0.89 &   9.5 &   12.8 &  3.2 \\
								& $10^{-2}$ &   0 &   0 &   29 &    3.68 &    2.73 &     0.32 &    3.57 &    2.40 &     0.89 &   9.5 &   12.8 &    3.2  \\
								& $10^{-1}$  &  78 &  99 &  100 &    0.18 &    0.00 &     0.00 &    0.06 &    0.01 &     0.00 &   9.5 &   12.8 &    3.2 \\
								$10^4$ & $10^{-3}$  &   1 &   3 &   85 &    1.33 &    1.06 &     0.00 &    2.54 &    1.88 &     0.15 &   3.4 &   4.1 &    0 \\
								& $10^{-2}$ &  22 &  24 &   85 &    0.63 &    0.77 &     0.00 &    0.87 &    0.57 &     0.15 &   3.4 &   4.1 &    0 \\
								& $10^{-1}$  & 100 & 100 &  100 &    0.00 &    0.00 &     0.00 &    0.00 &    0.00 &     0.00 &   3.4 &   4.1 &    0 \\
								$10^5$ & $10^{-3}$  &  17 &  55 &  100 &    0.24 &    0.06 &     0.00 &    1.10 &    0.44 &     0.00 &   1.2 &   0.2 &    0 \\
								& $10^{-2}$ & 100 &  99 &  100 &    0.00 &    0.00 &     0.00 &    0.00 &    0.01 &     0.00 &   1.2 &   0.2 &    0 \\
								& $10^{-1}$  & 100 & 100 &  100 &    0.00 &    0.00 &     0.00 &    0.00 &    0.00 &     0.00 &   1.2 &   0.2 &    0 \\
								\bottomrule
							\end{tabular}
						\end{table}

						\bibliographystyle{abbrvnat}

					\end{document}